\documentclass[11pt,english]{article}
\usepackage[utf8]{inputenc}
\usepackage{amsmath}
\usepackage{amsthm}
\usepackage{amssymb}
\usepackage{amsfonts}
\usepackage{comment}
\usepackage{mathtools}
\usepackage{algorithm}
\usepackage{graphicx}
\usepackage[font=footnotesize]{caption}
\usepackage{xcolor}
\usepackage{enumitem}
\usepackage{nicefrac}
\usepackage{bbding}

\usepackage{a4wide}
\usepackage{chngpage}
\usepackage{authblk}
\usepackage{gensymb}

\newtheorem{corollary}{Corollary}

\usepackage[T1]{fontenc}
\usepackage{babel}
\usepackage{threeparttable}
\usepackage{booktabs}
\usepackage{etoolbox}
\appto\TPTnoteSettings{\footnotesize}
\usepackage[font=footnotesize]{caption}
\usepackage{pifont}

\usepackage{mathtools, nccmath, textcomp}
\usepackage[round]{natbib}

\usepackage[backref=page]{hyperref}
\renewcommand*{\backref}[1]{}
\renewcommand*{\backrefalt}[4]{%
    \ifcase #1 %
    \or        (Cited on page~#2.)%
    \else      (Cited on pages~#2.)%
    \fi}

\usepackage{xcolor} 
\usepackage{amsfonts, amsthm, amssymb} 
\usepackage{mathtools}

\usepackage{subcaption}
\usepackage{makecell}
\usepackage{enumitem}
\usepackage{nicefrac}

\usepackage[normalem]{ulem}
\usepackage[normalem]{ulem}
\usepackage{algorithm}
\usepackage[noend]{algcompatible}
\usepackage{hyperref, tabularx}
\usepackage[capitalise]{cleveref}
\makeatletter
\newcommand{\multiline}[1]{%
  \begin{tabularx}{\dimexpr\linewidth-\ALG@thistlm}[t]{@{}X@{}}
    #1
  \end{tabularx}
}
\makeatother
\algnewcommand{\Initialize}[1]{%
  \State \textbf{Initialize:} #1}
\usepackage{nicefrac}
\DeclareMathOperator*{\argmin}{arg\,min}
\usepackage{amsmath}
\usepackage{amssymb}
\usepackage{amsfonts}
\usepackage{comment}
\usepackage{mathtools}
\usepackage{algorithm}
\usepackage{graphicx}

\usepackage{xcolor}
\usepackage{enumitem}

\usepackage{bbding}

\newcommand{\mbb}{\mathbb}

\newcommand{\mbe}{\mathbb E}

\newcommand{\lcb}{\left\{}
\newcommand{\rcb}{\right\}}
\newcommand{\lbr}{\left[}
\newcommand{\rbr}{\right]}

\newcommand{\bx}{{\mathbf x}}
\newcommand{\bxt}{{\mathbf x^{(t)}}}
\newcommand{\bxk}{{\mathbf x^{(k)}}}
\newcommand{\bxtp}{{\mathbf x^{(t+1)}}}

\newcommand{\be}{{\mathbf e}}
\newcommand{\bet}{{\mathbf e^{(t)}}}
\newcommand{\bek}{{\mathbf e^{(k)}}}

\newcommand{\bz}{{\mathbf z}}
\newcommand{\bzt}{{\mathbf z^{(t)}}}

\newcommand{\bPsi}{{\boldsymbol{\Psi}}}

\newcommand{\bPhi}{{\boldsymbol{\Phi}}}
\newcommand{\bPhit}{{\bPhi^{(t)}}}

\newcommand{\by}{{\mathbf y}}

\newcommand{\bq}{{\mathbf q}}

\newcommand{\G}{\nabla}

\usepackage{xcolor}

\usepackage{tikz}
\usetikzlibrary{calc,trees,positioning,arrows,chains,shapes.geometric,%
	decorations.pathreplacing,decorations.pathmorphing,shapes,%
	matrix,shapes.symbols}

\tikzstyle{startstop} = [rectangle, draw, rounded corners, align=center, minimum width=3cm, minimum height=1cm,text centered]
\tikzstyle{decision} = [diamond, draw, fill=blue!20, 
text width=4.5em, text badly centered, node distance=3cm, inner sep=0pt]
\tikzstyle{block} = [rectangle, draw, fill=blue!10, align=center, rounded corners, minimum width=3cm, minimum height=1cm]
\tikzstyle{blockcast} = [rectangle, draw, fill=red!10, align=center, rounded corners, minimum width=3cm, minimum height=0.45cm]
\tikzstyle{line} = [draw, -latex']
\tikzstyle{cloud} = [draw, ellipse,fill=red!20, node distance=3cm,
minimum height=2em]

\newcommand{\R}{\mathbb{R}}
\newcommand{\E}{\mathbb{E}}
\newcommand{\N}{\mathbb{N}}

\newcommand{\calN}{\mathcal{N}}
\newcommand{\bv}{\mathbf{v}}
\newcommand{\bX}{\mathbf{X}}

\newtheorem{example}{Example}

\makeatletter
\renewcommand\@fnsymbol[1]{}
\makeatother

\usepackage{multirow}
\usepackage{tabularx}

\newtheorem{remark}{Remark}

\newtheorem{assump}{Assumption}

\newtheorem{lemma}{Lemma}
\newtheorem{theorem}{Theorem}
\newtheorem{definition}{Definition}

\newcommand{\bi}{\mathbb{I}}

\footnotetext{The first two authors contributed equally.}

\author[1]{Aleksandar Armacki$^*$
}
\author[1]{Shuhua Yu$^*$}
\author[2]{Dragana Bajovi\'{c}}
\author[3]{Du\v{s}an Jakoveti\'{c}}
\author[1]{Soummya Kar}
\affil[1]{Carnegie Mellon University, Pittsburgh, PA, USA\\ \texttt{\{aarmacki,shuhuay,soummyak\}@andrew.cmu.edu }}
\affil[2]{Faculty of Technical Sciences, University of Novi Sad, Novi Sad, Serbia\\ \texttt{dbajovic@uns.ac.rs}}
\affil[3]{Faculty of Sciences, University of Novi Sad, Novi Sad, Serbia\\ \texttt{dusan.jakovetic@dmi.uns.ac.rs}}

\title{Large Deviation Upper Bounds and Improved MSE Rates of Nonlinear SGD: Heavy-tailed Noise and Power of Symmetry}
\date{}

\begin{document}

\maketitle

\begin{abstract}
    We study large deviation upper bounds and mean-squared error (MSE) guarantees of a general framework of nonlinear stochastic gradient methods in the online setting, in the presence of heavy-tailed noise. Unlike existing works that rely on the closed form of a nonlinearity (typically clipping), our framework treats the nonlinearity in a black-box manner, allowing us to provide unified guarantees for a broad class of bounded nonlinearities, including many popular ones, like sign, quantization, normalization, as well as component-wise and joint clipping. We provide several strong results for a broad range of step-sizes in the presence of heavy-tailed noise with symmetric probability density function, positive in a neighbourhood of zero and potentially unbounded moments. In particular, for non-convex costs we provide a large deviation upper bound for the minimum norm-squared of gradients, showing an asymptotic tail decay on an exponential scale, at a rate $\sqrt{t} / \log(t)$. We establish the accompanying rate function, showing an explicit dependence on the choice of step-size, nonlinearity, noise and problem parameters. Next, for non-convex costs and the minimum norm-squared of gradients, we derive the \emph{optimal} MSE rate $\widetilde{\mathcal{O}}(t^{-\nicefrac{1}{2}})$. Moreover, for strongly convex costs and the last iterate, we provide an MSE rate that can be made arbitrarily close to the \emph{optimal} rate $\mathcal{O}(t^{-1})$, improving on the state-of-the-art results in the presence of heavy-tailed noise. Finally, we establish almost sure convergence of the minimum norm-squared of gradients, providing an explicit rate, which can be made arbitrarily close to $o(t^{-\nicefrac{1}{4}})$.
\end{abstract}

\pagebreak

\section{Introduction}

Stochastic optimization is a well-studied problem, where the goal is to find a model that minimizes an expected cost, without having any knowledge of the underlying probability distribution. Formally, it can be stated as
\begin{equation}\label{eq:problem}
    \argmin_{\bx \in \R^d} \lcb f(\bx) \coloneqq \mbe_{\upsilon \sim \Upsilon} [\ell(\bx; \upsilon)] \rcb,
\end{equation} where $\bx \in \R^d$ represents model parameters, $\ell: \R^d \times \mathcal{V} \mapsto \R$ is a loss function, while $\upsilon \in \mathcal{V}$ is a random sample from the sample space $\mathcal{V}$, distributed according to the unknown probability distribution $\Upsilon$. The function $f: \mbb R^d \mapsto \mbb R$ is commonly known as the \emph{population cost}. The popularity of the formulation \eqref{eq:problem} stems from the fact that many learning problems, ranging from training generalized linear models, e.g., \cite{Shalev-Shwartz_sco}, to training neural networks, e.g., \cite{simsekli2019tail}, can be cast as instances of \eqref{eq:problem}.{\footnote{As can be seen in the Appendix, our proofs do not utilize the specific structure of the cost $f$ in \eqref{eq:problem}. We chose the form in \eqref{eq:problem}, as it is widely used and subsumes several popular learning paradigms, such as online learning and empirical risk minimization.}}

A popular method for solving \eqref{eq:problem} is the \emph{stochastic gradient descent} (SGD), which dates back to 1951 and the seminal work of Robbins and Monro \cite{robbins1951stochastic}. The idea of SGD is to update model parameters using only gradients of the loss {$\ell$} evaluated at random samples $\upsilon$. Equivalently, SGD can be seen as having access to noisy versions of gradients of $f$, a view taken by Robbins and Monro in their work. For a detailed review of (stochastic) gradient methods, see \cite{lan2020first,bottou2018optimization,Ablaev2024-polyak} and references therein. 

The theoretical guarantees of SGD-type methods have been studied extensively. For example, the MSE guarantees, which aim to establish the average behaviour across many sample paths of the algorithm, are typically derived under the assumption of \emph{uniformly bounded noise variance}, e.g., \cite{bach-sgd,rakhlin2012making,ghadimi2012optimal,ghadimi2013stochastic}. Additionally, works such as \cite{polyak1973pseudogradient,bottou2018optimization,khaled2022better}, consider generalizations of uniformly bounded variance assumption, by imposing different bounds on the \emph{second moment} of noise or stochastic gradients themselves. Another type of guarantees, such as large deviations and high-probability rates, study the tail behaviour of iterates generated by SGD-type algorithms, e.g., \cite{pmlr-v99-harvey19a,li2020high,liu2023high,hu2019diffusion,pmlr-v206-bajovic23a,azizian2024long}. The typical assumption in this line of work is \emph{sub-Gaussian} (i.e., \emph{light-tailed}) noise, e.g., \cite{nemirovski2009robust,rakhlin2012making,ghadimi2013stochastic,pmlr-v99-harvey19a,li-orabona,pmlr-v206-bajovic23a}. Finally, many works study the almost sure (a.s.) convergence of SGD-type algorithms, which guarantee convergence along every sample path of the algorithm, e.g., \cite{bertsekas-gradient,sgd-almost-sure,cevher-almost_sure,pmlr-v134-sebbouh21a}. 

While uniformly bounded variance and light-tailed noise are widely used in the literature, a growing body of works recently demonstrated that many modern learning models induce \emph{heavy-tailed} noise with \emph{unbounded second moment} (i.e., variance), e.g, \cite{zhang2020adaptive,heavy-tail-phenomena,simsekli2019tail}, highlighting a discrepancy between theoretical assumptions and empirical observations for many learning models. Instead, a more general assumption of \emph{uniformly bounded noise moments of order $p \in (1,2]$} is introduced, given by 
\begin{equation}\label{eq:bounded-moment}
    \E_{\upsilon \sim \Upsilon}\|\nabla \ell(\bx,\upsilon) - \nabla f(\bx) \|^p \leq \sigma^p, \tag{BM}
\end{equation} for every $\bx \in \R^d$ and some $p \in (1,2]$, $\sigma > 0$, e.g., \cite{zhang2020adaptive,liu2023breaking,liu2023stochastic,sadiev2023highprobability,improved-clipping}, subsuming the uniformly bounded variance as a special case when $p = 2$. When assumption \eqref{eq:bounded-moment} holds with $p<2$, the work \cite{zhang2020adaptive} shows that SGD diverges for $f(x) = x^2/2$ in the MSE sense for any step-size schedule, {while \cite{wang-conv-rates-infinite-var} show that SGD can converges for a sub-class of strongly convex functions, at a rate $\mathcal{O}(t^{c_p(1-p)})$, where $c_p > 0$ is a constant, possibly depending on $p$}. To control the excessive noise and ensure convergence in the general case, a clipping operator is introduced, resulting in \emph{clipped SGD} \cite{chen2020understanding,zhang2020adaptive,pmlr-v151-tsai22a,zhang2019clipping,liu2023breaking,liu2023stochastic,sadiev2023highprobability,improved-clipping}. Beyond reducing excessive noise, clipping is widely used to stabilize training \cite{pmlr-v28-pascanu13}, provide differential privacy \cite{chen2020understanding}, ensure convergence in the absence of classical smoothness \cite{zhang2019gradient} and facilitate distributed learning in the presence of adversaries \cite{shuhua-clipping}. However, clipped SGD is not the only nonlinear SGD-based method used in practice. For example, \emph{sign SGD} is used to accelerate the training of neural networks \cite{crawshaw2022general_signSGD} and drastically decrease the number of bits communicated \cite{bernstein2018signsgd,bernstein2018signsgd_iclr}, \emph{normalized SGD} is used to stabilize and speed up training \cite{hazan2015beyond,cutkosky20normalized_SGD}, with many variants of \emph{quantized SGD} routinely used to reduce the communication cost \cite{alistarh2017qsgd,gandikota2021vqsgd}. 

\begin{table}
\caption{\scriptsize MSE rates under heavy-tailed noise. Online indicates if the method is applicable in the online setting, i.e., by using an adaptive, time-varying step-size (indicated by lowercase $t$), or if it applicable in the offline setting only, i.e., requires a fixed step-size inversely proportional to a preset time horizon, which is optimized to obtain the best possible rate (indicated by uppercase $T$). The tilde symbol indicates factors poly-logarithmic in time $t$ are hidden in the asymptotic notation. It can be seen that we provide the optimal MSE rate for non-convex functions, matching that from \cite{chen2020understanding} and strictly better than the one from \cite{zhang2020adaptive} whenever $p < 2$, i.e., any heavy-tailed noise. For strongly convex costs, our rate is arbitrarily close to the optimal rate $\mathcal{O}(t^{-1})$, strictly better than the one from \cite{jakovetic2023nonlinear} and better than the one from \cite{sadiev2023highprobability} whenever $p < \nicefrac{2}{(1+\epsilon)}$, for any $\epsilon \in (0,\nicefrac{1}{2})$.}
\label{tab:mse}
\begin{adjustwidth}{-1in}{-1in} 
\begin{center}
\begin{threeparttable}
\begin{small}
\begin{sc}
\begin{tabular}{cccccc}
\toprule
\multicolumn{1}{c}{\rule{0pt}{2.5ex}\scriptsize Cost$^*$} & \multicolumn{1}{c}{\scriptsize Work} & \multicolumn{1}{c}{\scriptsize Nonlinearity} & \multicolumn{1}{c}{\scriptsize Noise} & \multicolumn{1}{c}{\scriptsize Online} & \multicolumn{1}{c}{\scriptsize Rate}\\
\midrule
\multirow{3.5}{*}{\scriptsize Non-convex} & \scriptsize \cite{zhang2020adaptive} & \scriptsize Clipping only & $\substack{\text{\scriptsize{bounded moment of}} \\ \text{\scriptsize{order }} p \in (1,2]}$ & \ding{55} & \scriptsize $\mathcal{O}\left(T^{\frac{2(1 - p)}{3p - 2}} \right)$ \\
\cmidrule{2-6}
& \scriptsize \cite{chen2020understanding}$^\circ$ & \scriptsize Clipping only & \multirow{2.5}{*}{$\substack{\text{\scriptsize{symmetric pdf,}} \\ \text{\scriptsize{positive around}} \\ \text{\scriptsize{zero}}}$} & \ding{55} & \scriptsize $\mathcal{O}\left(T^{-\nicefrac{1}{2}}\right)$ \\
& \scriptsize This paper & $\substack{\text{\scriptsize Component-wise} \\ \text{\scriptsize and joint}}$ & & \ding{52} & \scriptsize $\widetilde{\mathcal{O}}\left(t^{-\nicefrac{1}{2}}\right)^\star$ \\
\midrule
\multicolumn{1}{c}{\multirow{3.5}{*}{ $\substack{\text{\scriptsize Strongly} \\ \text{\scriptsize convex}}$}} & \scriptsize\cite{sadiev2023highprobability}$^\S$ & \scriptsize Clipping only & $\substack{\text{\scriptsize{bounded moment}} \\ \text{\scriptsize{of order }} p \in (1,2]}$ & \ding{55} & \scriptsize $\mathcal{O}\left(T^{\frac{2(1-p)}{p}} \right)$ \\ 
\cmidrule{2-6}
& \scriptsize \cite{jakovetic2023nonlinear} & \multirow{2}{*}{$\substack{\text{\scriptsize Component-wise} \\ \text{\scriptsize and joint}}$} & \multirow{2}{*}{$\substack{\text{\scriptsize{symmetric pdf,}} \\ \text{\scriptsize{positive around}} \\ \text{\scriptsize{zero}}}$} & \ding{52} & \scriptsize $\mathcal{O}\left(t^{-\zeta} \right)^\dagger$
\\
\multicolumn{1}{c}{} & \scriptsize This paper & & & \ding{52} & \scriptsize $\mathcal{O}\left(t^{-1 + \epsilon}\right)^\ddagger$ \\
\bottomrule
\end{tabular}
\end{sc}
\end{small}
\begin{tablenotes}\scriptsize
    \item[*] The metric for non-convex costs used in both \cite{zhang2020adaptive,chen2020understanding} and our work is given by $\min_{k \in [t]}\E\min\{\|\nabla f(\bxk) \|,\|\nabla f(\bxk) \|^2\}$. The metric used for strongly convex costs in all three works is the standard metric $\E[f(\bxt) - f^\star]$, where $f^\star = \min_{\bx \in \R^d}f(\bx)$ is the optimal value.
    
    \item[$\circ$] While the results from \cite{chen2020understanding} hold for some noises with non-symmetric PDF, their main insights come from considering symmetric noise. See the literature review for a more detailed discussion. 

    \item[$\star$] The MSE rate in the table is established for the step-size choice $\eta_t = \nicefrac{a}{\sqrt{t+1}}$. However, we provide guarantees for a wide range of step-sizes of the form $\eta_t = \nicefrac{a}{(t+1)^\delta}$, for any $\delta \in [\nicefrac{1}{2},1)$ and $a > 0$. 

    \item[$\S$] The authors in \cite{zhang2020adaptive} establish a MSE lower-bound $\Omega\left(T^{\nicefrac{2(1-p)}{p}}\right)$ for smooth, strongly convex functions, but show the lower-bound is attained by projected clipped SGD over a bounded domain and non-smooth cost. The authors in \cite{sadiev2023highprobability} show the lower-bound is attained for clipped SGD and smooth, strongly convex costs, however, in the high-probability sense. We include the results from \cite{sadiev2023highprobability}, as the MSE rate of the same order can be derived from their high-probability rate, by using the techniques developed in this paper. 
    
    \item[$\dagger$] The exponent $\zeta \in (0,1)$ depends on noise and choice of nonlinearity and constants like condition number and problem dimension. The authors in \cite{jakovetic2023nonlinear} show an asymptotic normality result, which, combined with a bound on the moment generating function, e.g., \cite{armacki2023high}, can be used to derive an asymptotic MSE bound with optimal rate $\mathcal{O}(t^{-1})$ and optimal dependence on problem constants.

    \item[$\ddagger$] The parameter $\epsilon \in (0,\nicefrac{1}{2})$ is user-specified and as such can be made arbitrarily small.
\end{tablenotes}
\end{threeparttable}
\end{center}
\vskip -0.1in
\end{adjustwidth}
\end{table}

\begin{table}
\caption{\scriptsize Long-term tail behaviour for non-convex costs in the presence of heavy-tailed noise. Decay rate represents the value $n_t$ such that $\limsup_{t \rightarrow \infty}n_t^{-1}\log\mathbb{P}(X_t > \theta) \leq -R_\theta$, for all $\theta  > 0$, where $X_t = \min_{k \in [t]}\|\nabla f(\bxk)\|^2$. Decay constant is the constant $R_\theta \geq 0$, capturing the dependence on different problem parameters, noise and choice of nonlinearity, possibly depending on $\theta$. The decay rate for \cite{armacki2023high,nguyen2023improved} is established from their finite-time high-probability bounds, whereas our decay rate is obtained by establishing a full large deviations principle upper bound, providing an explicit \emph{rate function}, see Sections \ref{sec:problem}, \ref{sec:main} for details. Since the long-term tail behaviour is asymptotic in nature, we do not include works that use a fixed step-size inversely proportional to a preset time horizon $T$ (e.g., \cite{sadiev2023highprobability}), as this would imply the algorithm uses a step-size equal to $0$ (as the rate is established for $T \rightarrow \infty$). The values $\sigma, p$ stem from the condition \eqref{eq:bounded-moment}, while $M_1 > 0$ is a global constant; $\alpha,\beta,\gamma$ are noise, nonlinearity and problem related constants, see Lemma \ref{lm:key-unified}; $a > 0$ is a user-specified step-size parameter, while $C$ and $L$ are the bound on nonlinearity and smoothness constant, see Assumptions \ref{asmpt:nonlin} and \ref{asmpt:L-smooth}, respectively. It can be seen that our tail decay rate is uniformly better than the decay rates resulting from \cite{armacki2023high,nguyen2023improved}. While the high-probability rate in \cite{nguyen2023improved} does not require a preset time horizon, it requires knowledge of problem parameters to tune the step-size and clipping radius, such as smoothness constant $L$, noise level $\sigma$ and moment $p$, which are typically unknown. On the other hand, our results and those from \cite{armacki2023high} are derived for general step-size schemes, requiring no knowledge of noise or problem parameters.}
\label{tab:ldp}
\begin{adjustwidth}{-1in}{-1in} 
\begin{center}
\begin{threeparttable}
\begin{small}
\begin{sc}
\begin{tabular}{cccc}
\toprule
\scriptsize Work & \scriptsize \cite{nguyen2023improved} & \scriptsize \cite{armacki2023high} & \scriptsize This paper \\
\midrule
\scriptsize Nonlinearity & \scriptsize Clipping only & \multicolumn{2}{c}{\scriptsize Component-wise and joint} \\
\midrule
\scriptsize Noise & $\substack{\text{\scriptsize bounded moment of} \\ \text{\scriptsize order $p \in (1,2]$}}$ &  \multicolumn{2}{c}{\scriptsize symmetric pdf, positive around zero}  \\
\midrule
\scriptsize Decay rate & \scriptsize ${\left(t^{\nicefrac{(p-1)}{p}} /\log(t) \right)^{\frac{2(p-1)}{(3p-2)}}}^*$  & \scriptsize $t^{\nicefrac{1}{4}}$ & \scriptsize $\sqrt{t}/\log(t)^\dagger$ \\
\midrule
\scriptsize Decay const. & \scriptsize $\frac{(L(f(\bx^{(1)}) - f^\star))^{\frac{(2-p)}{2p}}}{\sigma M_1^{\nicefrac{(p-1)}{p}}}\theta^{\frac{(p-1)}{p}}$ & \scriptsize $4a\min\{\alpha,\beta\}\min\{\sqrt{\theta},\theta\}$ & \scriptsize $\frac{\min\{\alpha^2,\beta^2\}}{16a^2C^4L^2}\min\{\sqrt{\theta},\theta\}$ \\
\bottomrule
\end{tabular}
\end{sc}
\end{small}
\begin{tablenotes}\scriptsize
    \item[*] Depending on the problem related constants and noise moment $p \in (1,2]$, there are two possible regimes for the decay rate and decay constant resulting from the high-probability bound in \cite{nguyen2023improved}. In the first regime, which occurs when $\theta$ is large (i.e., large deviations) or when the noise is extremely heavy-tailed (i.e., $\sigma$ large and $p \approx 1$), the decay rate and constant are given in the table. In the second regime, which occurs when $\theta$ is small (i.e., small deviations) or when the noise is well-behaved (i.e., $\sigma$ small and $p \approx 2$), the decay rate can be improved to $n_t = \left(\frac{t^{2(p-1)}}{ \log^{2p}(t)}\right)^{\nicefrac{1}{(3p-2)}}$, with the decay constant given by $R_\theta = \frac{\theta}{\sigma M_2\sqrt{L(f(\bx^{(1)}) - f^\star)}}$, where $M_2 > 0$ is a global constant. Note that even in the second regime, in which the decay rate from \cite{nguyen2023improved} is improved, the rate established in our work remains better. The first regime from \cite{nguyen2023improved} is listed in the table, as we are interested in the large deviation behaviour.  
    
    \item[$\dagger$] The decay rate in the table is established for the step-size choice $\eta_t = \nicefrac{a}{(t+1)^{\nicefrac{3}{4}}}$. However, we provide guarantees for a wide range of step-sizes of the form $\eta_t = \nicefrac{a}{(t+1)^\delta}$, for any $\delta \in (\nicefrac{1}{2},1)$ and $a > 0$.
\end{tablenotes}
\end{threeparttable}
\end{center}
\vskip -0.1in
\end{adjustwidth}
\end{table}

Although assumption~\eqref{eq:bounded-moment} helps bridge the gap between theory and practice, the downside is that the resulting convergence rates have exponents which explicitly depend on the noise moment $p$, and the bounds become vacuous as $p \rightarrow 1$.{\footnote{The convergence rates provided in works assuming \eqref{eq:bounded-moment}, both in high-probability and MSE sense, are of the form $\mathcal{O}(t^{c_p(1-p)})$, e.g., \cite{zhang2020adaptive,sadiev2023highprobability,nguyen2023improved,liu2023breaking}. In that sense, the guarantees become vacuous, as the exponent goes to zero when $p \rightarrow 1$. While different types of guarantees can be established, even in the presence of noise with moments of order $p \leq 1$, e.g., convergence in distribution in \cite{heavy-tail-phenomena}, our claim is made in relation to the guarantees considered in our and related works \cite{zhang2020adaptive,sadiev2023highprobability,nguyen2023improved,liu2023breaking}.}} This seems to contradict the strong performance of nonlinear SGD methods observed in practice and fails to explain the empirical success of nonlinear SGD during training of models such as neural networks in the presence of heavy-tailed noise. A growing body of works recently provided strong evidence that the stochastic noise during training of neural networks is \emph{symmetric},{\footnote{In the sense that the probability distribution/density function is symmetric around zero.}} by studying the empirical distribution of gradient noise and iterates generated during training, e.g., \cite{bernstein2018signsgd,bernstein2018signsgd_iclr,chen2020understanding,barsbey-heavy_tails_and_compressibility,pmlr-v238-battash24a}. Moreover, relying on a generalization of the central limit theorem, many works show theoretically that symmetric heavy-tailed noise, such as symmetric $\alpha$-stable distributions \cite{heavy-tail-book}, are appropriate noise models in many practical settings, e.g., when training neural networks with mini-batch SGD using a large batch size \cite{simsekli2019tail,pmlr-v108-peluchetti20b,heavy-tail-phenomena,barsbey-heavy_tails_and_compressibility}. On the other hand, works relying on assumption~\eqref{eq:bounded-moment} are inherently oblivious to this widely observed phenomena. The goal of this paper is to provide a unified study of the behaviour of nonlinear SGD methods in the presence of symmetric heavy-tailed noise and the benefits this additional noise structure brings.

\paragraph{Literature review} We now review the literature on nonlinear SGD methods, focusing on works assuming either \eqref{eq:bounded-moment} or symmetric (heavy-tailed) noise. 

\paragraph{Guarantees under assumption \eqref{eq:bounded-moment}} The authors in \cite{zhang2020adaptive} note that the gradient noise during training of BERT model resembles a Levy distribution and introduce assumption \eqref{eq:bounded-moment}. They establish a MSE lower complexity bound $\Omega\left(T^{\nicefrac{(2-2p)}{(3p-2)}} \right)$\footnote{We use uppercase $T$ to indicate an offline method, i.e., one that requires a fixed step-size depending on a preset time horizon $T$, while lowercase $t$ indicates an online method, which does not require a preset time horizon and allows for a time-varying step-size. The distinction is important, as the rates obtained for offline methods require optimizing the step-size choice, making it dependent on various unknown problem related parameters and preset time horizon $T$. Additionally, online methods can be applied in the offline settings (when all the data is available beforehand), whereas offline methods requiring a fixed step-size inversely proportional to a preset time horizon can not be applied in the online settings such as learning on continuous streams of data, as it precludes the use of a preset time horizon (or implies using a step-size equal to zero for the horizon $T = \infty$.)} for first-order methods and non-convex costs, showing it is attained by clipped SGD. The authors in \cite{liu2023stochastic} study clipped SGD for non-smooth convex costs in both MSE and high-probability sense, showing that it achieves the lower bounds $\widetilde{\Omega}\left(t^{\nicefrac{(1-p)}{p}}\right)$\footnote{The tilde symbol indicates factors poly-logarithmic in time $t$ (or failure probability for high-probability guarantees) are hidden in the asymptotic notation.} and $\widetilde{\Omega}\left(t^{\nicefrac{(2-2p)}{p}} \right)$ for convex and strongly convex costs. The work \cite{liu2023breaking} studies clipped SGD with momentum for non-convex costs, showing it achieves an accelerated high-probability rate $\widetilde{\mathcal{O}}\left(T^{\nicefrac{(2-2p)}{(2p-1)}} \right)$. The work \cite{sadiev2023highprobability} studies high-probability guarantees of clipped SGD for a broad range of stochastic optimization and variational inequality problems, notably showing it achieves the optimal rate $\widetilde{\mathcal{O}}\left(T^{\nicefrac{(2-2p)}{p}} \right)$ for strongly convex costs. Work \cite{nguyen2023improved} shows that clipped SGD achieves the optimal high-probability rates $\widetilde{\mathcal{O}}\left(t^{\nicefrac{(1-p)}{p}} \right)$ and $\widetilde{\mathcal{O}}\left(t^{\nicefrac{(2-2p)}{(3p-2)}} \right)$ for smooth convex and non-convex costs.

\paragraph{Guarantees under symmetric noise} The work \cite{bernstein2018signsgd_iclr} studies sign SGD under noise with symmetric probability density function (PDF) and component-wise bounded variance, showing a MSE rate $\mathcal{O}\left(T^{-\nicefrac{1}{2}} \right)$ for non-convex costs. The authors in \cite{chen2020understanding} analyze clipped SGD, by bounding the difference between the true noise and a user-specified symmetric one. If the original noise is symmetric,\footnote{The authors in \cite{chen2020understanding} show that their results hold for some non-symmetric noises, like positively-skewed or mixture of symmetric noises. However, the main insight comes from analyzing the behaviour of clipped SGD under symmetric noise.} they show the optimal MSE rate $\mathcal{O}\left(T^{-\nicefrac{1}{2}} \right)$ for non-convex costs, without imposing any moment requirements on the noise.\footnote{The optimality of rate from \cite{chen2020understanding} does not contradict the optimality of the one in \cite{zhang2020adaptive}, as they are derived under different assumptions. The rate $\mathcal{O}\left(T^{-\nicefrac{1}{2}}\right)$ obtained in \cite{chen2020understanding} is optimal for stochastic first-order methods and non-convex costs under noise with uniformly bounded variance and is recovered by \cite{zhang2020adaptive} when $p = 2$. However, the result in \cite{chen2020understanding} is stronger, showing that the optimal rate holds for any symmetric noise (and some non-symmetric ones), regardless of noise moments.} The work \cite{jakovetic2023nonlinear} studies a broad framework of nonlinear SGD methods, including clipping, sign and quantization, in the presence of noise with symmetric PDF, positive in a neighbourhood of zero and bounded first moment. For strongly convex costs they show almost sure convergence, asymptotic normality (i.e., \emph{small deviations}) and MSE convergence with rate $\mathcal{O}\left(t^{-\zeta} \right)$, where $\zeta \in (0,1)$ depends on the choice of nonlinearity, noise and problem parameters. The authors in \cite{armacki2023high} study high-probability guarantees of the same nonlinear framework and noise from \cite{jakovetic2023nonlinear}, noting that the moment requirement can be relaxed to include noises with unbounded first (absolute) moment. They show improved rates with \emph{constant exponents}, achieving rates $\widetilde{\mathcal{O}}\left(t^{-\nicefrac{1}{4}}\right)$, for both non-convex and strongly convex costs.\footnote{The authors in \cite{armacki2023high} also show high-probability rates of the last iterate for strongly convex costs, of order $\widetilde{\mathcal{O}}\left(t^{-\zeta^\prime} \right)$ for general nonlinearities (including component-wise and joint ones, like clipping and normalization). Here $\zeta^\prime \in (0,1)$ is of the same order as $\zeta$ from \cite{jakovetic2023nonlinear}, i.e., $\zeta^\prime = \Theta(\zeta)$.} 

\paragraph{Contributions} Works relying on assumption~\eqref{eq:bounded-moment} are inherently oblivious to the widely observed noise symmetry, providing bounds that explicitly depend on noise moments and vanish as $p \rightarrow 1$, failing to properly characterize the strong performance of SGD-type methods observed in practice. Works like \cite{chen2020understanding,jakovetic2023nonlinear,armacki2023high} exploit noise symmetry in their analysis and provide improved guarantees, however, they come with downsides. For example, while \cite{chen2020understanding} provide the optimal MSE rate for clipped SGD, their analysis relies on the closed-form expression of the clipping operator and as such can not be directly extended to other nonlinearities, like sign or quantization. Moreover, \cite{chen2020understanding} analyze only the MSE behaviour in the offline setting, providing no guarantees on the tail behaviour or the online setting.\footnote{While MSE guarantees imply a tail bound via Chebyshev's inequality, the resulting bound is loose, only providing polynomially decaying tails, rather than exponentially decaying ones.} On the other hand, while \cite{jakovetic2023nonlinear,armacki2023high} provide guarantees for a broad framework of nonlinearities, their bounds are not tight, with \cite{armacki2023high} showing exponential tail decay with a sub-optimal decay rate arbitrarily close to $t^{\nicefrac{1}{4}}$, while the MSE rate exponent $\zeta \in (0,1)$ from \cite{jakovetic2023nonlinear} decays as the problem dimension and condition number grow, resulting in a loose bound for many problems. We aim to fill this gap, by providing a unified study and sharp characterization of the behaviour of nonlinear SGD methods in the presence of symmetric heavy-tailed noise, with potentially unbounded moments. {It is also worth mentioning \cite{gorbunov2023breaking}, who combine clipping with a median-of-means gradient estimator in the presence of heavy-tailed noise with potentially unbounded moments, showing that it is possible to achieve the optimal high-probability rates for convex and strongly convex costs in the offline setting. Our work is different, in that we mainly focus on non-convex costs in the online setting, i.e., the time horizon is unknown and we only require access to a single stochastic gradient per iteration.} Our contributions are as follows. 
\begin{enumerate}
    \item We study the long-term tail probability behaviour and MSE guarantees for a broad class of nonlinear SGD methods in the online setting and the presence of heavy-tailed noise with symmetric PDF, positive around zero, with potentially unbounded moments. In the considered setting, the nonlinearity is treated in a black-box manner, allowing us to establish guarantees for a broad class of nonlinear mappings, including sign, quantization, normalization, component-wise and joint clipping. We provide strong guarantees for both non-convex and strongly convex costs.

    \item To characterize the long-term tail probability behaviour, we establish a large deviation principle (LDP) upper bound for the minimum norm-squared of gradients for non-convex costs. We show that the tail probabilities decay asymptotically at an exponential scale, with a decay rate $\sqrt{t} / \log(t)$. The accompanying rate function fully characterizes the dependence on  the noise, choice of nonlinearity and other problem parameters. In Table \ref{tab:ldp} we provide a detailed comparison of asymptotic tail decay rates with works \cite{armacki2023high,nguyen2023improved}. It can be seen that the tail decay rate in our work is better than the rates resulting from \cite{armacki2023high,nguyen2023improved}. Moreover, we establish a full rate function, allowing us to characterize the probability of belonging to more general sets compared to \cite{armacki2023high,nguyen2023improved}, providing a tight characterization of the large deviation behaviour.

    \item Next, we establish MSE guarantees, showing the \emph{optimal} rate $\widetilde{\mathcal{O}}\left(t^{-\nicefrac{1}{2}} \right)$ for non-convex costs and a rate arbitrarily close to the \emph{optimal} $\mathcal{O}\left(t^{-1}\right)$ for strongly convex costs. In Table \ref{tab:mse} we provide a detailed comparison of MSE rates with \cite{zhang2020adaptive,chen2020understanding} for non-convex and \cite{jakovetic2023nonlinear} for strongly convex costs. Compared to \cite{zhang2020adaptive}, whose rate exponents depend on noise moment $p$, our rate exponents are constant and strictly better whenever $p < 2$, i.e., for any heavy-tailed noise. Compared to \cite{chen2020understanding}, who derive the same MSE rate for clipping only in the offline setting, our rates apply to a broader range of nonlinearities and the online setting. For strongly convex costs and the last iterate, our rate is strictly better than the rate from \cite{jakovetic2023nonlinear} and better than the rate achievable under \eqref{eq:bounded-moment} whenever $p < 2$. The improved rate is derived using a simple technique that shows how exponentially decaying tail probabilities can be used to establish MSE guarantees, which is of independent interest.

    \item Additionally, we provide a.s. convergence guarantees for minimum norm-squared of gradients for non-convex costs and the last iterate for strongly convex costs. Moreover, we explicitly characterize the a.s. convergence rate, showing a rate arbitrarily close to $o(t^{-\nicefrac{1}{4}})$ for non-convex costs. For strongly convex costs we establish a rate $o(t^{-\tau})$, {where $\tau > 0$ is a problem related constant (see Section \ref{sec:main} ahead for details)}, complementing the result from \cite{jakovetic2023nonlinear}, by providing a rate of convergence in the a.s. sense. To the best of our knowledge, \emph{this is the first a.s. convergence result with an explicit rate in the presence of heavy-tailed noise}.

    \item Our work relaxes various conditions and improves on the current literature in several ways. Compared to \cite{zhang2020adaptive,nguyen2023improved,sadiev2023highprobability}, who only study clipping under the \eqref{eq:bounded-moment} assumption, we study a much broader class of nonlinearities, utilizing the widely observed symmetry of noise, requiring no moment assumptions and providing sharper rates independent of noise moments. Compared to \cite{chen2020understanding}, who study MSE guarantees in the offline setting of clipping under noise symmetry, we provide optimal MSE rates in the online setting for a much broader range of nonlinearities, as well as a sharp characterization of the tail decay rate. Compared to \cite{armacki2023high}, who provide strong guarantees for the same nonlinear framework, we establish sharper long-term tail behaviour, allow for a broader range of step-sizes and provide MSE and almost sure guarantees. Compared to \cite{jakovetic2023nonlinear}, who only study strongly convex costs, we study non-convex costs, improve the MSE guarantees for strongly convex costs, and provide a sharp characterization of the long-term tail behaviour.
\end{enumerate}

\paragraph{Technical challenges} There are several challenges toward establishing our results. First, unlike existing works, which rely on the closed form expression of the nonlinear mapping, we treat the nonlinearity in a black-box fashion. To that end, we provide a general result characterizing the behaviour of the nonlinearity with respect to the true gradient, in Lemma \ref{lm:key-unified}. Next, establishing the LDP upper bound requires careful control of the moment-generating function (MGF) and a generalized version of the Gartner-Ellis theorem, which we provide in Lemma \ref{lm:gartner-ellis}. Finally, providing the almost sure convergence rates requires tight control of the MGF of a surrogate sequence, facilitated by a careful analysis of the resulting MGF for different step-size schedules.

\paragraph{Other applications} While the main goal of this paper is addressing the heavy-tailed noise, there are many applications where nonlinear SGD arises as the natural choice. For example, clipping is widely used to ensure differential privacy and provide security in distributed learning in the presence of adversarial agents. Similarly, sign and other quantizers are used to alleviate the communication cost that arises in distributed learning. As such, the framework considered in this paper has a broad spectrum of applications and our work complements the existing studies, by providing strong theoretical guarantees of the performance of nonlinear SGD in many applications of interest, when the stochastic noise is heavy-tailed.     

\paragraph{Paper organization} The rest of the paper is organized as follows. Section \ref{sec:problem} presents the preliminaries, Section \ref{sec:method} describes the nonlinear SGD framework, Section \ref{sec:main} presents the main results and Section \ref{sec:conclusion} concludes the paper. Appendix contains results omitted from the main body. The remainder of the section introduces notation.

\paragraph{Notation} We denote real numbers and vectors by $\R$ and $\R^d$. The set of positive integers is denoted by $\N$. For $m \in \N$, we denote the set of positive integers up to and including $m$ by $[m]$, i.e., $[m] = \{1,\ldots,m\}$. Regular and bold symbols denote scalars and vectors, i.e., $x \in \R$ and $\bx \in \R^d$. Euclidean inner product and norm are denoted by $\langle \cdot,\cdot\rangle$ and $\|\cdot\|$. Notation $o(\cdot)$, $\mathcal{O}(\cdot)$, $\Omega(\cdot)$ and $\Theta(\cdot)$ is the standard asymptotic notation, i.e., for two non-negative sequences $\{a_n\}_{n \in \N}$, $\{b_n\}_{n \in \N}$, $a_n = o(b_n)$ ($a_n = \mathcal{O}(b_n)$, respectively $a_n = \Omega(b_n)$) if $\lim_{n \rightarrow \infty}\frac{a_n}{b_n} = 0$ ($\limsup_{n \rightarrow \infty}\frac{a_n}{b_n} < \infty$, respectively $\liminf_{n\rightarrow \infty}\frac{a_n}{b_n} > 0$), while $a_n = \Theta(b_n)$ if and only if $a_n = \mathcal{O}(b_n)$ and $a_n = \Omega(b_n)$.

\section{Preliminaries}\label{sec:problem}

In this section we describe the methodology used in the paper. Let $\{X_t \}_{t \in \N}$ denote a sequence of non-negative random variables generated by applying a stochastic method to solve \eqref{eq:problem}. Here, $X_t$ can be any quantity of interest, e.g., $X_t = \|\nabla f(\bxt)\|$ for non-convex, or $X_t = \|\bxt - \bx^\star\|$ for strongly convex costs. \emph{MSE guarantees} aim to quantify the behaviour of the quantity $\E\left[X_t^2\right]$. As such, the MSE guarantees describe the average performance across many runs. While MSE guarantees offer an important indicator of performance, different guarantees are required when it becomes infeasible to perform many runs of an algorithm, as is the case in many modern learning problems, e.g., in deep learning.

\emph{Large deviations} aim to quantify the tail decay of rare events on an exponential scale. In particular, for a sequence of random variables $\{X_t\}_{t \in \N}$ and any measurable set $A \subset \R$, the goal is to find a lower semi-continuous function $I: \R \mapsto [0,\infty]$, so that
\begin{equation}\label{eq:LDP}
    -\inf_{x \in A^{\degree}}I(x) \leq \liminf_{t \rightarrow \infty}\frac{1}{n_t}\log \mathbb{P}_t(A) \leq \limsup_{t \rightarrow \infty}\frac{1}{n_t}\log \mathbb{P}_t(A) \leq -\inf_{x \in \overline{A}}I(x), 
\end{equation} where $\mathbb{P}_t(A) \coloneqq \mathbb{P}(X_t \in A)$, while $A^{\degree}$, $\overline{A}$ denote the topological interior and closure of $A$ and $n_t$ is a positive sequence such that $\lim_{t \rightarrow \infty}n_t = \infty$, referred to as the \emph{decay rate}. The function $I$ is called the \emph{rate function} and is said to be a \emph{good} rate function if all of its sub-level sets are compact. When \eqref{eq:LDP} holds, the sequence $X_t$ is said to satisfy a LDP with rate function $I$.\footnote{It is sometimes said that $X_t$ satisfies a LDP when \eqref{eq:LDP} holds with $n_t = t$, while if the rate is slower, e.g., $n_t = \sqrt{t}$, it is said that $X_t$ satisfies a \emph{moderate deviation principle}, e.g., \cite{dembo2009large}. However, we follow the convention from \cite{ellis-ld} and say that $X_t$ satisfies a LDP if \eqref{eq:LDP} holds for any rate $n_t$.} If only the right-hand side inequality in \eqref{eq:LDP} is satisfied, $X_t$ satisfies a LDP upper bound. Large deviation theory has a long history \cite{varadhan,ellis-ld,dembo2009large}, with applications in various areas, such as physics, distributed detection and inference, e.g, \cite{TOUCHETTE20091,ellis-ld,braca-ld,bajovic-detection-ld,bajovic-detection-ld2,bajovic-inference-ld}. In the context of SGD methods and machine learning, large deviation guarantees have been studied in \cite{hu2019diffusion,pmlr-v206-bajovic23a,azizian2024long,lindhe-large,pmlr-v242-jongeneel24a,masegosa-large}. However, these studies focus on the behaviour of linear SGD under light-tailed noise.  

Large deviations are closely related to high-probability guarantees, where the goal is to quantify the behaviour of the tail event $\{\omega \in \Omega: \: X_t(\omega) > \theta\}$, for any $\theta > 0$, typically on an exponential scale, e.g., \cite{lan2012optimal,li-orabona,sadiev2023highprobability,nguyen2023improved}. While high-probability guarantees focus on the finite-time (i.e., short-term) regime, large deviation studies often provide a more precise characterization of the long-term tail behaviour, by establishing sharp asymptotic guarantees, e.g., \cite{pmlr-v206-bajovic23a,azizian2024long}. We utilize the large deviations approach, with the aim of providing a tight characterization of the long-term tail behaviour of the iterates of nonlinear SGD in the presence of heavy-tailed noise.

\section{Nonlinear SGD Framework}\label{sec:method}

To solve \eqref{eq:problem} in the online setting and in the presence of heavy-tailed noise, a nonlinear SGD-based framework is deployed, consisting of the following steps. First, a deterministic initial model $\bx^{(1)} \in \mbb R^d$ and a nonlinear map $\boldsymbol{\Psi}:\mbb R^d \mapsto \mbb R^d$ are chosen.\footnote{While the initial model is deterministically chosen, it can be any vector in $\R^d$. This distinction is required for the theoretical analysis in the next section.} Next, in iteration $t = 1,2,\ldots$, a random sample $\upsilon^{(t)}$ is observed and the gradient of the loss $\ell$ at the current model $\bxt$ and sample $\upsilon^{(t)}$ is evaluated.\footnote{Equivalently, we have access to a first-order oracle that streams gradients of $\ell$, instead of samples.} Then, the current model is updated, according to the rule
\begin{equation}\label{eq:update}
    \bxtp = \bxt - \eta_t\mathbf{\Psi}\left(\nabla \ell(\bxt;\upsilon^{(t)})\right),
\end{equation} where $\eta_t > 0$ is the step-size at iteration $t$, and the process is then repeated. We make the following assumption on the nonlinear map $\bPsi$.

\begin{assump}\label{asmpt:nonlin}
The nonlinear map $\bPsi: \mbb R^d \mapsto \mbb R^d$ is either of the form $\bPsi(\bx) = \bPsi(x_1,\dots,x_d) = \lbr \calN_1(x_1), \dots, \calN_1(x_d) \rbr^\top$ (i.e., \emph{component-wise}) or $\bPsi(\bx) = \bx\calN_2(\|\bx\|)$ (i.e., \emph{joint}), where the mappings $\calN_1,\: \calN_2: \R \mapsto \R$ satisfy
\begin{enumerate}
    \item $\calN_1,\calN_2$ are continuous almost everywhere (with respect to the Lebesgue measure), with $\calN_1$ piece-wise differentiable, while the mapping $a \mapsto a\calN_2(a)$ is non-decreasing.
    \item $\calN_1$ is monotonically non-decreasing and odd, while $\calN_2$ is non-increasing.
    \item $\calN_1$ is either discontinuous at zero, or strictly increasing on $(-c_1,c_1)$, for some $c_1 > 0$, with $\calN_2(a) > 0$, for any $a > 0$.
    \item $\calN_1$ and $\bx\calN_2(\|\bx\|)$ are uniformly bounded, i.e., for some $C_1,\: C_2 > 0$, and all $x \in \R$, $\bx \in \R^d$, we have $|\calN_1(x)| \leq C_1$ and $\|\bx\calN_2(\|\bx\|)\|\leq C_2$.
\end{enumerate}
\end{assump}

The general form of Assumption \ref{asmpt:nonlin} allows us to treat the nonlinearity in a black-box manner, only relying on its general properties and not the closed-form expression. Note that the fourth property implies $\|\bPsi(\bx)\| \leq C_1\sqrt{d}$ for component-wise and $\|\bPsi(\bx)\|\leq C_2$ for joint nonlinearities. To facilitate a unified presentation, we will use the general bound $\|\bPsi(\bx)\| \leq C$ and specialize where appropriate. Assumption \ref{asmpt:nonlin} is satisfied by a wide class of nonlinearities, including many popular ones, such as
\begin{enumerate}
    \item \emph{Sign}: $[\bPsi(\bx)]_i = \text{sign}(x_i)$, for each component $i = 1,\ldots,d$;
    \item \emph{Component-wise clipping}: $[\bPsi(\bx)]_i = x_i$, for $|x_i| \leq m$, and $[\bPsi(\bx)]_i = m\cdot\text{sign}(x_i)$, for $|x_i| > m$, $i = 1,\ldots,d$, where $m > 0$ is a user-specified constant;
    \item \emph{Quantization}: for each $i = 1,\ldots,d$, let $[\bPsi(\bx)]_i = r_j$, for $x_i \in (q_j,q_{j+1}]$, with $j = 0,\ldots,J-1$ and $-\infty = q_0 < q_1 <\ldots < q_J = +\infty$, where $r_j$'s and $q_j$'s are chosen such that each component of $\bPsi$ is an odd function, and we have $\max_{j \in \{0,\ldots,J-1\}}|r_j| < R$, for some user-specified constant $R > 0$;
    \item \emph{Clipping}: $\bPsi(\bx) = \min\left\{1, \nicefrac{M}{\|\bx\|}\right\}\bx$, for some user-specified constant $M > 0$;
    \item \emph{Normalization}: $\bPsi(\bx) = \nicefrac{\bx}{\|\bx\|}$ if $\bx \neq \mathbf{0}$, otherwise $\bPsi(\bx) = \mathbf{0}$.
\end{enumerate}

\section{Main Results}\label{sec:main}

In this section we present the main results of the paper. Subsection \ref{subsec:prelim} outlines the assumptions, while Subsection \ref{subsec:theory-nonconv} presents the main results. Proofs of results from this section can be found in the Appendix.

\subsection{Assumptions}\label{subsec:prelim}

In this section we state the assumptions used in our work. 
{
\begin{assump}\label{asmpt:L-smooth}
The cost $f$ is bounded from below, with $L$-Lipschitz gradients, i.e., $\inf_{\bx \in \R^d}f(\bx) > -\infty$ and $\|\nabla f(\bx) - \nabla f(\by) \| \leq L\|\bx - \by\|$, for all $\bx,\by \in \R^d$.
\end{assump}

\begin{remark}
    Boundedness from below and Lipschitz continuous gradients are widely used and standard for non-convex costs, e.g., \cite{ghadimi2013stochastic,cevher-almost_sure,madden2020high,armacki2023high_old,armacki2023high}. 
\end{remark}
}
\begin{remark}
    Lipschitz gradients imply the widely used smoothness inequality, i.e., $f(\by) \leq f(\bx) + \langle\nabla f(\bx),\by - \bx\rangle + \frac{L}{2} \|\bx - \by \|^2$ for all $\bx, \by \in \R^d$, see, e.g., \cite{bertsekas-gradient,nesterov-lectures_on_cvxopt,Wright_Recht_2022}.
\end{remark}

We denote the infimum of $f$ by $f^\star \coloneqq \inf_{\bx \in \R^d}f(\bx)$. In addition to Assumption \ref{asmpt:L-smooth}, we will sometimes use the following assumption.

\begin{assump}\label{asmpt:cvx}
The population cost is $\mu$-strongly convex, i.e., for some $\mu > 0$ and every $\bx, \by \in \R^d$, we have $f(\by) \geq f(\bx) + \langle \nabla f(\bx), \by - \bx \rangle + \frac{\mu}{2}\|\by - \bx\|^2$. 
\end{assump}

{Defining the stochastic noise at iteration $t$ as $\bzt \coloneqq \nabla \ell(\bxt;\upsilon^{(t)}) - \G f(\bxt)$,} we can rewrite the update \eqref{eq:update} as
\begin{align}\label{eq:nonlin-sgd}
    \bxtp = \bxt - \eta_t \boldsymbol{\Psi}(\G f(\bxt) + \bzt). 
\end{align} To simplify the notation, we introduce the shorthand $\bPsi^{(t)} \coloneqq \boldsymbol{\Psi}(\nabla f(\bxt) + \bzt)$. {Define the natural filtration as $\mathcal{F}_1 = \{\emptyset,\Omega\}$ and $\mathcal{F}_t \coloneqq \sigma\left(\{\bx^{(2)},\ldots,\bx^{(t)}\}\right)$, for $t \geq 2$.\footnote{Recall that in our setup, the initialization $\bx^{(1)} \in \R^d$ is an arbitrary, but deterministic quantity.}} We make the following assumption on the sequence of noise vectors $\{\bzt \}_{t \in \N}$.

\begin{assump}\label{asmpt:noise}
The noise vectors $\{\bzt \}_{t \in \N}$ are independent and identically distributed, with a symmetric PDF $P$, which is strictly positive around zero, i.e.,  $P(-\bz) = P(\bz)$, for all $\bz \in \R^d$ and $P(\bz) > 0$, for all $\|\bz\| \leq B_0$ and some $B_0 > 0$. {Moreover, the noise at time $t$ is independent of history up to $t$, i.e., $\E[\bzt \: \vert \: \mathcal{F}_t] = \E[\bzt]$}. 
\end{assump}

\begin{remark}
    Assumption \ref{asmpt:noise} imposes mild requirements, satisfied by many \linebreak distributions, such as Gaussian, or a broad class of symmetric heavy-tailed $\alpha$-stable distributions \cite{stable-distributions,csimcsekli2019heavy,heavy-tail-book}. As discussed in the introduction, symmetric heavy-tailed noise was widely observed during training of deep learning models, across datasets, architectures and batch sizes \cite{bernstein2018signsgd,bernstein2018signsgd_iclr,chen2020understanding,barsbey-heavy_tails_and_compressibility,pmlr-v238-battash24a}. Building on the generalized CLT, works \cite{simsekli2019tail,pmlr-v108-peluchetti20b,heavy-tail-phenomena,barsbey-heavy_tails_and_compressibility} provide theoretical justification for this phenomena, e.g., when training neural nets with a large batch size, effectively ``symmetrizing'' the noise.
\end{remark}

For a fixed vector $\bx \in \R^d$, define the function $\bPhi(\bx) \coloneqq \mbe_{\bz} [\bPsi (\bx + \bz)] = \int \bPsi(\bx+\bz) P(\bz) d\bz$,\footnote{If $\bPsi$ is a component-wise nonlinearity, then $\bPhi$ is a vector with components $\phi_i(x_i) = \mbe_{z_i}[\calN_1(x_i + z_i)]$, where $\E_{z_i}$ is the marginal expectation with respect to the $i$-th noise component, $i \in [d]$.} where the expectation is taken with respect to the random vector $\bz$. We use the shorthand $\bPhit \coloneqq \E\left[\bPsi (\nabla f(\bxt) + \bzt) \: \vert \: \mathcal{F}_t \right]$.\footnote{Conditioning on $\mathcal{F}_t$ ensures that the quantity $\nabla f(\bxt)$ is deterministic and $\bPhit$ is well defined.} The vector $\bPhit$ represents the ``denoised'' version of $\bPsi^{(t)}$. We can rewrite the update rule \eqref{eq:nonlin-sgd} as
\begin{equation}\label{eq:nonlin-sgd2}
    \bxtp = \bxt - \eta_t\bPhit + \eta_t\bet,    
\end{equation} where $\bet \coloneqq \boldsymbol{\Phi}^{(t)} - \boldsymbol{\Psi}^{(t)}$, represents the \emph{effective noise}. Noting that $\bet$ is bounded (since both $\bPsi^{(t)}$ and $\bPhit$ are), it readily follows that the effective noise is light-tailed, even though the original noise may not be, e.g., \cite{vershynin_2018,jin2019short}. Further properties of the effective noise vectors are provided in Appendix \ref{app:proofs}.

\subsection{Guarantees}\label{subsec:theory-nonconv}

In this subsection we provide the guarantees of the proposed framework. We begin with a result on the behaviour of the ``denoised'' direction $\bPhi(\bx)$.

\begin{lemma}\label{lm:key-unified}
    Let Assumptions \ref{asmpt:nonlin} and \ref{asmpt:noise} hold. Then, for any $\bx \in \R^d$, we have $\langle \bPhi(\bx),\bx\rangle \geq \min\left\{\alpha\|\bx\|,\beta\|\bx\|^2 \right\}$, where $\alpha,\beta > 0$ are noise, nonlinearity and problem dependent constants.
\end{lemma}

Lemma \ref{lm:key-unified} is a key technical result, which characterizes the behaviour of the inner product between $\bPhi(\bx)$ and $\bx$ and facilitates the main results of the paper. We specialize the constants $\alpha$ and $\beta$ for component-wise and joint nonlinearities in Appendix \ref{app:proofs}. Next, define $N_t \coloneqq \sum_{k = 1}^{t}\eta_k$ and $\widetilde{\eta}_k \coloneqq \eta_kN_t^{-1}$, $k \in [t]$, so that $\sum_{k = 1}^{t}\widetilde{\eta}_k = 1$. Our first result characterizes the long-term behaviour of the quantity $X_t \coloneqq \min_{k \in [t]}\|\nabla f(\bxk)\|^2$ for non-convex costs.

\begin{theorem}\label{thm:non-conv}
    Let Assumptions \ref{asmpt:nonlin}, \ref{asmpt:L-smooth} and \ref{asmpt:noise} hold. Let $\{\bxt\}_{t \in \N}$ be the sequence generated by \eqref{eq:nonlin-sgd}, with step-size $\eta_t = \frac{a}{(t + 1)^\delta}$, for any $\delta \in (\nicefrac{1}{2},1)$ and $a > 0$. Then, the sequence $\{X_t\}_{t \in \N}$, where $X_t = \min_{k \in [t]}\|\nabla f(\bxk)\|^2$, satisfies a LDP upper bound with a decay rate and good rate function given as follows.
    \begin{enumerate}
        \item For $\delta \in (\nicefrac{1}{2},\nicefrac{3}{4})$, the decay rate is $n_t = t^{2\delta - 1}$, with the rate function given by
    \begin{displaymath}
        I_1(x) = \begin{cases} 
        \frac{(3-4\delta)\min\{\alpha^2,\beta^2\}}{16a^2C^4L^2}\min\{\sqrt{x},x\}, & x \geq 0 \\
        +\infty, & x < 0
    \end{cases}.
    \end{displaymath} 

    \item For $\delta = \nicefrac{3}{4}$, the decay rate is $n_t = \nicefrac{\sqrt{t}}{\ln(t)}$, with the rate function given by
    \begin{displaymath}
        I_2(x) = \begin{cases}
            \frac{\min\{\alpha^2,\beta^2\}}{16a^2C^4L^2}\min\{\sqrt{x},x\}, & x \geq 0 \\
        +\infty, & x < 0    
        \end{cases}    
    \end{displaymath}

    \item For $\delta \in (\nicefrac{3}{4},1)$, the decay rate is $n_t = t^{2(1-\delta)}$, with the rate function given by
    \begin{displaymath}
        I_3(x) = \begin{cases} 
        \frac{(2\delta-1)(4\delta-3)\min\{\alpha^2,\beta^2\}}{16C^2\left[(1-\delta)^2(4\delta-3)\|\nabla f(\bx^{(1)}\|^2 + a^2(2\delta - 1)L^2C^2\right]}\min\{\sqrt{x},x\}, & x \geq 0 \\
        +\infty, & x < 0
    \end{cases}.
    \end{displaymath}
    \end{enumerate}
\end{theorem}

\begin{remark}
    Functions $I_j$, $j \in [3]$, are lower-semi continuous, making them well-defined rate functions. Moreover, it can be seen that all three have compact sub-level sets, making them good rate functions. For any $x < 0$, we have $I_j(x) = +\infty$, $j \in [3]$. Recalling \eqref{eq:LDP}, this readily implies that for any set $A \subset (-\infty,0)$, the probability of $X_t = \min_{k \in [t]}\|\nabla f(\bxk)\|^2$ belonging to $A$ is zero, which is to be expected, as $X_t \geq 0$. 
\end{remark}

\begin{remark}
    Notice that, as $\delta \rightarrow \nicefrac{3}{4}$, we have $2\delta - 1 = 2(1 - \delta) = \nicefrac{1}{2}$, therefore for both step-sizes regimes choosing $\delta = \nicefrac{3}{4} \pm \nicefrac{\epsilon}{2}$, for any $\epsilon \in (0,\nicefrac{1}{2})$, the decay rate is $n_t = t^{\nicefrac{1}{2}-\epsilon}$, which can be made arbitrarily close to $\sqrt{t}$, for $\epsilon$ small. This however comes at the expense of a multiplicative factor $\epsilon$ in the rate functions $I_1, I_3$.
\end{remark}

\begin{remark}
    It can be seen that the rate functions $I_j$, $j \in [3]$, depend on the noise and choice of nonlinearity through values $\alpha,\beta$ and $C$, problem related parameters such as $L$, step-size choice through $a,\delta$ and in the case of $I_3$, {the stationarity gap $\|\nabla f(\bx^{(1)})\|^2$}. While $I_3$ provides the tightest dependence on problem parameters, by accounting for the initial model {stationarity gap through $\|\nabla f(\bx^{(1)})\|^2$}, it can be seen that $I_3(x) \leq I_1(x) \leq I_2(x)$,\footnote{While $I_1$ and $I_3$ are defined for different ranges of $\delta$ and are therefore not directly comparable, it can be seen that the values $(3-4\delta)$ and $(4-3\delta)$ are symmetric around $\nicefrac{3}{4}$ on the interval $(\nicefrac{1}{2},1)$. Noticing this and the fact that $I_3(x) \leq \frac{(4\delta-3)\min\{\alpha^2,\beta^2\}x^2}{16a^2C^4L^2}$, it follows that $I_3(x) \leq I_1(x)$ for symmetrically chosen deltas, e.g., by choosing $\delta_1 = \nicefrac{3}{4} - \epsilon$ and $\delta_2 = \nicefrac{3}{4} + \epsilon$, for any $\epsilon \in (0,\nicefrac{1}{4})$.} and since the decay rate achieved for $\delta = \nicefrac{3}{4}$ is the fastest, our theory suggests the sharpest long-term tail decay is achieved for $\delta  = \nicefrac{3}{4}$.
\end{remark} 

{
\begin{remark}
    While the metric $X_t = \min_{k \in [t]}\|\nabla f(\bx^{(k)})\|^2$ can be difficult to compute in practice, due to not knowing the true gradient or the index that minimizes the gradient norm, it is widely used as a measure of performance in non-convex optimization, see, e.g., \cite{ghadimi2013stochastic,nguyen2023improved,cevher-almost_sure,pmlr-v134-sebbouh21a}. Next, while Theorem \ref{thm:non-conv} is stated using the metric $X_t$, the results are actually established on the quantity $\sum_{k = 1}^t\widetilde{\eta}_kG_k$, where $G_k = \min\{\|\nabla f(\bxk)\|, \|\nabla f(\bxk)\|^2 \}$, which is then related to $X_t$, see the proof in Appendix for details. As noted in \cite{ghadimi2013stochastic}, the quantity $\sum_{k = 1}^t\widetilde{\eta}_kG_k$ is the expectation with respect to randomly selecting an index $k \in [t]$ (or randomly choosing an iterate $\bx^{(k)}$), with the probability of choosing $k$ equal to $\widetilde{\eta}_k$. In practice, this is equivalent to running nonlinear SGD with a random stopping time, following the distribution induced by $\{\widetilde{\eta}_k\}_{k \in [t]}$. In that sense, Theorem \ref{thm:non-conv} provides guarantees on the (average) performance of randomly stopped nonlinear SGD. This is important, as random stopping is practically implementable, removing the need for evaluating the true gradient or finding the best iterate. Moreover, same guarantees hold for the uniformly stopped method, i.e., $\frac{1}{t}\sum_{k = 1}^tG_k$, which follows by noting that the step-sizes are decreasing, hence $\widetilde{\eta}_t\sum_{k = 1}G_k \leq \sum_{k = 1}\widetilde{\eta}_kG_k$, and dividing both sides by $t\widetilde{\eta}_t$.  
\end{remark}

If we additionally assume that the cost $f$ is strongly convex, then a LDP upper bound can be established for the quantity $\|\widehat{\bx}^{(t)} - \bx^\star\|^2$, where $\widehat{\bx}^{(t)} \coloneqq \frac{1}{t}\sum_{k = 1}^t\bxk$ is the Polyak-Ruppert average of iterates, e.g., \cite{ruppert,polyak-ruppert}. The formal result is stated next.

\begin{corollary}\label{cor:str-cvx}
    Let Assumptions \ref{asmpt:nonlin}, \ref{asmpt:L-smooth}, \ref{asmpt:cvx} and \ref{asmpt:noise} hold. Let $\{\bxt\}_{t \in \N}$ be the sequence generated by \eqref{eq:nonlin-sgd}, with step-size $\eta_t = \frac{a}{(t + 1)^\delta}$, for any $\delta \in (\nicefrac{1}{2},1)$ and $a > 0$, and let $h: [0,\infty) \mapsto [0,\infty)$ be given by $h(x) = x/2$ if $x \leq 1/\mu^2$, otherwise $h(x) = \sqrt{x}/\mu - 1/(2\mu^2)$. Then, the sequence $\{X_t\}_{t \in \N}$, where $X_t = \|\widehat{\bx}^{(t)} - \bx^\star\|^2$, satisfies a LDP upper bound with a decay rate and good rate function given as follows.
    \begin{enumerate}
        \item For $\delta \in (\nicefrac{1}{2},\nicefrac{3}{4})$, the decay rate is $n_t = t^{2\delta - 1}$, with the rate function given by
    \begin{displaymath}
        I_1(x) = \begin{cases} 
        \frac{(3-4\delta)\mu^4\min\{\alpha^2,\beta^2\}}{64a^2C^4L^2}h(x), & x \geq 0 \\
        +\infty, & x < 0
    \end{cases}.
    \end{displaymath} 

    \item For $\delta = \nicefrac{3}{4}$, the decay rate is $n_t = \nicefrac{\sqrt{t}}{\ln(t)}$, with the rate function given by
    \begin{displaymath}
        I_2(x) = \begin{cases}
            \frac{\mu^4\min\{\alpha^2,\beta^2\}}{64a^2C^4L^2}h(x), & x \geq 0 \\
        +\infty, & x < 0    
        \end{cases}    
    \end{displaymath}

    \item For $\delta \in (\nicefrac{3}{4},1)$, the decay rate is $n_t = t^{2(1-\delta)}$, with the rate function given by
    \begin{displaymath}
        I_3(x) = \begin{cases} 
        \frac{(2\delta-1)(4\delta-3)\mu^4\min\{\alpha^2,\beta^2\}}{64C^2\left[(1-\delta)^2(4\delta-3)\|\nabla f(\bx^{(1)}\|^2 + a^2(2\delta - 1)L^2C^2\right]}h(x), & x \geq 0 \\
        +\infty, & x < 0
    \end{cases}.
    \end{displaymath}
    \end{enumerate}
\end{corollary}

While providing sharper results for strongly convex costs is of interest, e.g., establishing an exponential decay (i.e., decay rate $n_t = t$) or providing an LDP upper bound for the last iterate, it requires a different approach, e.g., akin to \cite{pmlr-v206-bajovic23a} and is left for future work. Next, note that the rate functions in Theorem \ref{thm:non-conv} depend on noise, problem dimension and choice of nonlinearity through $\alpha$, $\beta$ and $C$. To illustrate this dependence further, we next provide an example with a specific noise distribution and derive the values of the said constants, for different choices of nonlinearities. 

\begin{example}
    Consider the noise with independent, identically distributed \linebreak components, whose PDF is given by $P(\bz) = \rho(z_1)\times\rho(z_2)\ldots\rho(z_d)$, where $\rho(z) = \frac{\kappa-1}{2(|z|+1)^{\kappa}}$, for some $\kappa > 2$. It can be shown that the noise only has finite moments of order less than $\kappa - 1$, see \cite{jakovetic2023nonlinear}, implying it is heavy-tailed whenever $\kappa \in (2,3]$. Next, consider the sign and normalization nonlinearities. We know from the previous section that $C_s = \sqrt{d}$ and $C_n = 1$. As shown in Lemma \ref{lm:key-extended} in the Appendix, we have the following values of $\alpha$ and $\beta$: $\alpha_s = \nicefrac{\phi^\prime(0)\xi}{2\sqrt{d}}$, $\beta_s = \nicefrac{\phi^\prime(0)}{2d}$ and $\alpha_n = p_0\calN_2(1)/2$, $\beta_n = p_0\calN_2(1)$, where $\phi^\prime(0), \xi > 0$ depend on the choice of nonlinearity and noise, with $p_0 = P(\mathbf{0})$. For the specific noises and nonlinearities, it can be shown that $\phi^\prime(0) = \kappa-1$, $\xi \approx \frac{1}{\kappa}$, $\calN_2(1) = 1$ and $P(\mathbf{0}) = \big[\frac{\kappa-1}{2}\big]^d$, see \cite{armacki2023high,jakovetic2023nonlinear} for details. Combining everything and omitting other problem related constants for simplicity, we have the following rate functions, in terms of their dependence on $\alpha$, $\beta$ and $C$
    \begin{align*}
        I_s(x) &= \frac{\min\{\nicefrac{(\kappa-1)}{2\kappa\sqrt{d}}, \nicefrac{(\kappa-1)}{2d}\}^2}{d^2}\min\{\sqrt{x},x\} = \frac{(\kappa-1)^2}{4d^4}\min\{\sqrt{x},x\}, \\
        I_n(x) &= \frac{\min\big\{[\nicefrac{(\kappa-1)}{2}]^d/2, [\nicefrac{(\kappa-1)}{2}]^d \big\}^2}{1}\min\{\sqrt{x},x\} = \frac{1}{4}\bigg[\frac{\kappa-1}{2} \bigg]^{2d}\min\{\sqrt{x},x\},
    \end{align*} where for $I_s$ we used the fact that in general, we can expect $\kappa \ll \sqrt{d}$. We can now notice a couple of things. First, the rate function is uniformly bounded away from zero with respect to $\kappa$, even as $\kappa \rightarrow 2$. Next, while the rate function for sign goes to zero at a polynomial rate as the dimension grows, we can notice three regimes for normalization: 1) if $\kappa > 3$ (i.e., bounded variance), then the rate function grows with the problem dimension; 2) if $\kappa = 3$, then the rate function is independent of the problem dimension; 3) if $\kappa \in (2,3)$ (i.e., the heavy-tailed regime), then the rate function goes to zero at an exponential rate as the dimension grows, implying that sign has a better dependence on problem dimension in the heavy-tailed regime. This is consistent with a result from \cite{zhang2020adaptive}, which shows that component-wise clipping has better dependence on problem dimension than joint clipping, for some heavy-tailed noises. Same can be shown to hold for the noise from this example, namely that component clipping has better dependence on dimension than the joint one, see \cite{armacki2023high} for details.  
\end{example}
}

We next present the MSE convergence guarantees.

\begin{theorem}\label{thm:mse-ft}
    Let Assumptions \ref{asmpt:nonlin}, \ref{asmpt:L-smooth} and \ref{asmpt:noise} hold. Let $\{\bxt\}_{t \in \N}$ be the sequence generated by \eqref{eq:nonlin-sgd}, with step-size $\eta_t = \frac{a}{(t + 1)^\delta}$, for any $\delta \in [\nicefrac{1}{2},1)$ and $a > 0$. Then, for all $t \geq 1$, the following holds.
    \begin{enumerate}
        \item If $\delta = \nicefrac{1}{2}$, then $\min_{k \in [t]}\E\min\left\{\|\nabla f(\bxk)\|,\|\nabla f(\bxk)\|^2\right\} = \widetilde{\mathcal{O}}((t+1)^{-\nicefrac{1}{2}})$.

        \item If $\delta \in (\nicefrac{1}{2},1)$, then $\min_{k \in [t]}\E\min\left\{\|\nabla f(\bxk)\|,\|\nabla f(\bxk)\|^2\right\} = \mathcal{O}((t+1)^{\delta - 1})$.
    \end{enumerate}
\end{theorem}

\begin{remark}
     For the choice $\delta = \nicefrac{1}{2}$, Theorem \ref{thm:mse-ft} shows that nonlinear SGD achieves the \emph{optimal} MSE rate $\widetilde{\mathcal{O}}\left( t^{-\nicefrac{1}{2}}\right)$ for non-convex functions, for \emph{any nonlinearity} and \emph{any symmetric heavy-tailed noise}.{\footnote{While the optimality of rate $\widetilde{\mathcal{O}}(t^{-1/2})$ was established in \cite{Arjevani2023} for $\E[\min_{k \in [t]}\|\nabla f(\bx^{(k)})\|^2]$, we can expect that $\E[Z_t] = \E[\min_{k \in [t]}\|\nabla f(\bx^{(k)})\|^2]$ for $t$ sufficiently large, where $Z_t = \min_{k \in [t]}\min\{\|\nabla f(\bx^{(k)})\|,\|\nabla f(\bx^{(k)})\|^2\}$, seeing that $\E[Z_t] \rightarrow 0$, as $t \rightarrow \infty$. In that sense, the rate $\widetilde{\mathcal{O}}(t^{-1/2})$ is also optimal for $\E[Z_t]$.}} This result is on par with results for clipping in \cite{zhang2020adaptive} under \ref{eq:bounded-moment} for $p = 2$ and \cite{chen2020understanding} under symmetric noise, however, our results are significantly more general. Compared to \cite{chen2020understanding,zhang2020adaptive}, who establish convergence of the same metric for clipping only in the offline setting, our result applies to a much broader range of nonlinearities and is applicable in the online regime (i.e., does not require fixed step-size and time horizon). Additionally, compared to \cite{zhang2020adaptive}, our rate is strictly better whenever $p < 2$, i.e., for any heavy-tailed noise. 
\end{remark}

Theorem \ref{thm:mse-ft} can be leveraged to show MSE guarantees for the more standard metric $X_t = \min_{k \in [t]}\|\nabla f(\bxk)\|^2$. The formal result is stated next.

{
\begin{corollary}\label{cor:mse}
    Let Assumptions \ref{asmpt:nonlin}, \ref{asmpt:L-smooth} and \ref{asmpt:noise} hold. Let $\{\bxt\}_{t \in \N}$ be the sequence generated by \eqref{eq:nonlin-sgd}, with step-size $\eta_t = \frac{a}{(t + 1)^\delta}$, for any $\delta \in (\nicefrac{2}{3},\nicefrac{3}{4})$ and $a > 0$. We then have $\E[X_t] = \mathcal{O}(t^{\delta - 1})$, for all $t \geq \max\left\{\sqrt[1-\delta]{\frac{8(1-\delta)}{a\beta}\Big(f(\bx^{(1)})-f^\star + a^2C^2(L/2 + 8\|\nabla f(\bx^{(1)})\|^2) \Big)},\sqrt[3\delta-2]{\frac{64a^3C^4L^2}{\beta(1-\delta)(3-4\delta)}} \right\}$.
\end{corollary}

\begin{remark}
    Corollary \ref{cor:mse} implies a MSE rate arbitrarily close to $\mathcal{O}(t^{-1/3})$, which is sub-optimal. This sub-optimality stems from the constraint on the step-size parameter $\delta > 2/3$ imposed in Theorem 1 in \cite{armacki2023high}, which is used as an intermediary result in our proof. While improving the result from Theorem 1 in \cite{armacki2023high} to allow for $\delta = 1/2$ would lead to the optimal MSE rate in Corollary \ref{cor:mse}, it is beyond the scope of the current paper and is left for future work.  
\end{remark}

\begin{remark}
    Theorems \ref{thm:non-conv} and \ref{thm:mse-ft} subsume general convex costs as a special case, establishing guarantees in terms of the stationarity gap. However, providing results for convex costs in terms of the optimality gap $f(\bx) - f^\star$ is highly non-trivial, as our analysis is built on Lemma \ref{lm:key-unified}, which can not be directly related to the optimality gap. Analyzing the general convex case would require a fundamentally different approach, e.g., by quantifying the behaviour of $\langle \bPhi(\nabla f(\bx)),  \bx - \bx^\star \rangle$, which is beyond the scope of the current paper and is left for future work.
\end{remark}
}
If in addition the cost is strongly convex, we have the following improved result.

\begin{theorem}\label{thm:cvx-mse}
    Let Assumptions \ref{asmpt:nonlin}, \ref{asmpt:L-smooth}, \ref{asmpt:cvx} and \ref{asmpt:noise} hold. Let $\{\bxt\}_{t \in \N}$ be the sequence generated by \eqref{eq:nonlin-sgd}, with step-size $\eta_t = \frac{a}{(t + 1)^\delta}$, for any $\delta \in (\nicefrac{1}{2},1)$ and $a > 0$. Then, one has $\E [f(\bxt) - f^\star] = \mathcal{O} (t^{-\delta})$. Choosing $\delta = 1 - \epsilon$, for any $\epsilon \in (0,\nicefrac{1}{2})$, we get the near-optimal MSE rate $\E [f(\bxt) - f^\star] = \mathcal{O} (t^{-1 + \epsilon})$.
\end{theorem}

\begin{remark}
    The work \cite{polyak1984criterial} provides an asymptotic $\mathcal{O}(1/t)$ MSE rate assuming $\|\nabla f(\bx) \|^2 \ge 2\mu (f(\bx) - f^*) + o(f(\bx) - f^*)$, which is weaker than strongly convexity. However, \cite{polyak1984criterial} imposes stronger assumptions on the nonlinearity and noise, requiring the existence of some $\chi > 0$, such that $\bx^\top \bPhi'(0)\bx \ge \chi \|\bx\|^2$, for all $\bx \in \R^d$. We now present an example where this condition fails to hold, while Assumption \ref{asmpt:nonlin} is satisfied. Consider the normalization operator $ \bPsi(\bx) = \bx / \| \bx \|$, which satisfies Assumption \ref{asmpt:nonlin}. {It can be shown that in this case $\bPhi'(0) = \int (1/\| \bz \|)  \big( I - \bz \bz^\top/{\| \bz \|^2} \big) p(\bz) d\bz.$ Suppose that the noise vector $\bz$ is given by $\bz = g\bz_0$, where $\bz_0 \in \R^d$ is a fixed vector, while $g \sim \mathcal{N}(0,1)$ is a standard normal random variable. Then, the matrix $I - \bz \bz^\top/\| \bz \|^2$ is rank-deficient and $\bPhi'(0)$ is not positive definite, violating the assumption from \cite{polyak1984criterial}.} Additionally, our MSE rate is finite-time, whereas the rate from \cite{polyak1984criterial} is asymptotic.
\end{remark} 

Finally, we present the almost sure convergence guarantees. Recall the quantity $Z_t = \min_{k \in [t]}\min\{\|\nabla f(\bxk)\|,\|\nabla f(\bxk)\|^2 \}$. We then have the following result.

\begin{theorem}\label{thm:ncvx-as}
    Let Assumptions \ref{asmpt:nonlin}, \ref{asmpt:L-smooth} and \ref{asmpt:noise} hold. Let $\{\bxt\}_{t \in \N}$ be the sequence generated by \eqref{eq:nonlin-sgd}, with step-size $\eta_t = \frac{a}{(t + 1)^\delta}$, for any $\delta \in (\nicefrac{1}{2},1)$ and $a > 0$. Then, for all $0 < \epsilon < \min\{\delta - \nicefrac{1}{2}, 1 - \delta\}$, it holds that $(t+1)^{\min\{\delta-\nicefrac{1}{2}, 1-\delta\}-\epsilon} Z_t \xrightarrow{\text{a.s.}} 0$.  If in addition Assumption \ref{asmpt:cvx} holds, then there exists a problem dependent constant $\nu > 0$, such that for any $\tau < \min \{2a \mu \nu, 2\delta - 1\}$, we have $(t + 1)^{\tau} (f(\bxt) - f^\star) \xrightarrow{\text{a.s.}} 0$. 
\end{theorem}

\begin{remark}
    Setting $\delta = \nicefrac{3}{4}$ in the first part of Theorem \ref{thm:ncvx-as} results in the rate $Z_t = o(t^{-\nicefrac{1}{4}+\epsilon})$. The constant $\nu$ from the second part of Theorem \ref{thm:ncvx-as} is a problem, noise and nonlinearity dependent constant. The final rate exponent $\tau$ is related to the MSE rate exponent $\zeta$ from \cite{jakovetic2023nonlinear}. While sub-optimal, it allows us to explicitly quantify the a.s. convergence rate of the last iterate for strongly convex costs.
\end{remark}

\section{Conclusion}\label{sec:conclusion}

We perform a comprehensive study of the behaviour of a broad class of nonlinear SGD algorithms in the online setting and the presence of heavy-tailed noise. In our analysis the nonlinearity is treated in a black-box manner, allowing us to subsume several popular nonlinearities, such as sign, quantization, normalization, component-wise and joint clipping. We establish several strong results for both non-convex and strongly convex costs in the presence of noise with symmetric PDF, positive in a neighbourhood of zero, while making no moment assumptions. For non-convex costs and norm-squared of gradient of the best iterate, we establish a sharp upper bound on the long-term tail behaviour, demonstrating an exponential decay, with a rate $\sqrt{t}/\log(t)$, as well as the optimal MSE rate $\widetilde{\mathcal{O}}\left(t^{-\nicefrac{1}{2}}\right)$. Moreover, for strongly convex costs and the last iterate we show an improved MSE rate, arbitrarily close to the optimal rate $\mathcal{O}\left(t^{-1} \right)$. Crucially, our results provide sharp rates with constant exponents, independent of noise and problem parameters. Compared to state-of-the-art, our results are either better or of the same order, while simultaneously analyzing a much broader framework of nonlinearities and allowing for relaxed noise moment conditions. Finally, we show the quantities of interest converge almost surely and derive explicit convergence rates for both non-convex and strongly convex costs.

\bibliography{bibliography}

\appendix

\section{Introduction} The Appendix contains proofs and intermediate results used in our work. Appendix \ref{app:proofs} provides proofs omitted from the main body. Appendix \ref{app:garnter-ellis} provides a proof of the Gartner-Ellis theorem for a general decay rate.

\section{Missing Proofs}\label{app:proofs} In this section we provide the proofs omitted from the main body. Subsection \ref{subsec:proof-lm-error} provides some properties of the effective noise. Subsection \ref{subsec:proof-lm-key} proves Lemma \ref{lm:key-unified}, Subsection \ref{subsec:proof-thm-nonconv} proves Theorem \ref{thm:non-conv} and Corollary \ref{cor:str-cvx}, Subsection \ref{subsec:proofs-mse} proves Theorems \ref{thm:mse-ft}, \ref{thm:cvx-mse} and Corollary \ref{cor:mse}, while Subsection \ref{subsec:proofs-as} proves Theorem \ref{thm:ncvx-as}.

\subsection{Properties of effective noise}\label{subsec:proof-lm-error}

Prior to stating the properties of the effective noise, we define the concept of sub-Gaussianity.

\begin{definition}\label{def:subgaus}
    A zero-mean random vector $\mathbf{v} \in \R^d$ is sub-Gaussian, if there exists a constant $N > 0$, such that for any $\bx \in \R^d$, we have $\E\left[e^{\langle\bx,\mathbf{v}\rangle}\right] \leq \exp\left(\nicefrac{N\|\bx\|^2}{2} \right).$
\end{definition}

We are now ready to state the properties of effective noise vectors $\{\bet\}_{t \in \N}$.

\begin{lemma}\label{lm:error_component}
    If Assumptions \ref{asmpt:nonlin}, \ref{asmpt:noise} hold, the effective noise $\{\bet\}_{t \in \mbb N}$ satisfies 
    \begin{enumerate}
        \item $\E[\bet\vert \: \mathcal{F}_t] = 0$  and  $\|\bet\| \leq 2C$,
        \item The effective noise is sub-Gaussian, with constant $N = 8C^2$.
    \end{enumerate}
\end{lemma}

\begin{proof}
    Part 1. follows from the definitions of $\bet$, $\bPsi^{(t)}$, $\bPhit$ and Assumption \ref{asmpt:nonlin}. Part 2. follows from part 1 and Hoeffding's inequality.
\end{proof}

\subsection{Proof of Lemma \ref{lm:key-unified}}\label{subsec:proof-lm-key}

Prior to proving Lemma \ref{lm:key-unified}, we state two results. The first result, due to \cite{polyak-adaptive-estimation}, provides some properties of the mapping $\bPhi$ for component-wise nonlinearities under symmetric noise.

\begin{lemma}\label{lm:polyak-tsypkin}
    Let Assumptions \ref{asmpt:nonlin} and \ref{asmpt:noise} hold, with the nonlinearity $\bPsi: \R^d \mapsto \R^d$ being component-wise, i.e., of the form $\bPsi(\bx) = \begin{bmatrix} \calN_1(x_1),\ldots,\calN_1(x_d)\end{bmatrix}^\top$. Then, the map $\bPhi: \R^d \mapsto \R^d$ is of the form $\bPhi(\bx) = \begin{bmatrix} \phi_1(x_1),\ldots,\phi_d(x_d) \end{bmatrix}^\top$, where $\phi_i(x_i) = \E_{z_i}\left[\calN_1(x_i + z_i)\right]$ is the marginal expectation of the $i$-th noise component and it holds:
    \begin{enumerate}
        \item $\phi_i$ is non-decreasing and odd, with $\phi_i(0) = 0$;
        \item $\phi_i$ is differentiable at zero, with $\phi_i^\prime(0) > 0$.
    \end{enumerate}
\end{lemma}

The second result, due to \cite{jakovetic2023nonlinear}, gives a property of $\bPhi$ for joint nonlinearities.

\begin{lemma}\label{lm:jakovetic-joint}
    Let Assumption \ref{asmpt:nonlin} hold, with the nonlinearity $\bPsi: \R^d \mapsto \R^d$ being joint, i.e., of the form $\bPsi(\bx) = \bx\calN_2(\|\bx\|)$. Then for any $\bx, \bz \in \R^d$ such that $\|\bz\| > \|\bx\|$
    \begin{displaymath}
        \left|\calN_2(\|\bx + \bz\|) - \calN_2(\|\bx - \bz\|)\right| \leq \nicefrac{\|\bx\|}{\|\bz\|}\left[\calN_2(\|\bx + \bz\|) + \calN_2(\|\bx - \bz\|) \right].
    \end{displaymath} 
\end{lemma}

Next, define $\phi^\prime(0) \coloneqq \min_{i \in [d]}\phi_i^\prime(0)$ and $p_0 \coloneqq P(\mathbf{0})$. We are now ready to prove Lemma \ref{lm:key-unified}. For convenience, we restate the full lemma below. 

\begin{lemma}\label{lm:key-extended}
    Let Assumptions \ref{asmpt:nonlin} and \ref{asmpt:noise} hold. Then, for any $\bx \in \R^d$, we have $\langle \bPhi(\bx),\bx\rangle \geq \min\left\{\alpha\|\bx\|,\beta\|\bx\|^2 \right\}$, where $\alpha,\beta > 0$ are noise, nonlinearity and problem dependent constants. If the nonlinearity $\bPsi$ is component-wise, we have $\alpha = \nicefrac{\phi^\prime(0)\xi}{2\sqrt{d}}$ and $\beta = \nicefrac{\phi^\prime(0)}{2d}$, where $\xi > 0$ is a constant that depends only on the noise and choice of nonlinearity. If $\bPsi$ is joint, then $\alpha = p_0\calN_2(1) / 2$ and $\beta = p_0\calN_2(1)$.
\end{lemma}

\begin{proof}
    If $\bPsi$ is a component-wise nonlinearity, using Lemma \ref{lm:polyak-tsypkin} it follows that, for any $x \in \R$, and any $i \in [d]$, we have $\phi_i(x) = \phi_i(0) + \phi_i^\prime(0)x + h_i(x)x = \phi_i^\prime(0)x + h_i(x)x$, where $h_i: \R \mapsto \R$ is such that $\lim_{x \rightarrow 0}h_i(x) = 0$. From the definition of $\phi^\prime(0)$ and the fact that $\phi^\prime(0) > 0$, it follows that there exists a $\xi > 0$ (depending only on the nonlinearity $\mathcal{N}_1$) such that, for each $x \in \R$ and all $i \in [d]$, we have $|h_i(x)| \leq \phi^\prime(0) / 2$, if $|x| \leq \xi$. Therefore, for any $0 \leq x \leq \xi$, we have $\phi_i(x) \geq \frac{\phi^\prime(0)x}{2}$. On the other hand, for $x > \xi$, since $\phi_i$ is non-decreasing, we have from the previous relation that $\phi_i(x) \geq \phi_i(\xi) \geq \frac{\phi^\prime(0)\xi}{2}$. Therefore, it follows that $\phi_i(x) \geq \frac{\phi^\prime(0)}{2}\min\{x,\xi \}$, for any $x \geq 0$. Combined with the oddity of $\phi_i$, we get $x\phi_i(x) = |x|\phi_i(|x|) \geq \frac{\phi^\prime(0)}{2}\min\{\xi|x|,x^2\}$, for any $x \in \R$. Using the previously established relations, we have, for any $\bx \in \R^d$
    \begin{align*}
        \langle \bx, \bPhi(\bx) \rangle &= \sum_{i = 1}^d |x_i|\phi(|x_i|) \geq \max_{i \in [d]}|x_i|\phi_i(|x_i|) \geq \frac{\phi^\prime(0)}{2}\max_{i \in [d]}\min\{\xi|x_i|,|x_i|^2\} \\ &= \frac{\phi^\prime(0)}{2}\min\{\xi\|\bx\|_{\infty},\|\bx\|_{\infty}^2\} \geq \frac{\phi^\prime(0)}{2}\min\left\{\frac{\xi\|\bx\|}{ \sqrt{d}},\frac{\|\bx\|^2}{d}\right\},
    \end{align*} where the last inequality follows from the fact that $\|\bx\|_{\infty} \geq \|\bx\| / \sqrt{d}$. If $\bPsi$ is a joint nonlinearity, the first part of the proof follows a similar idea to the one in \cite[Lemma 6.2]{jakovetic2023nonlinear}, with some differences due to differences in noise assumptions. Fix an arbitrary $\bx \in \R^d \setminus \{\mathbf{0}\}$. By the definition of $\bPsi$, we have
    \begin{align*}
        \langle \bPhi(\bx),\bx \rangle = \int_{\bz \in \R^d}\underbrace{(\bx + \bz)^\top \bx \calN_2(\|\bx + \bz\|)}_{\eqqcolon M(\bx,\bz)}P(\bz)d\bz = \int_{\{\bz \in \R^d: \langle\bz,\bx\rangle \geq 0\} \cup \{\bz \in \R^d: \langle\bz,\bx\rangle < 0\}}\hspace{-10em}M(\bx,\bz)P(\bz)d\bz. 
    \end{align*} Next, by the symmetry of $P$, it readily follows that $\langle \bPhi(\bx),\bx \rangle = \int_{J_1(\bx)}M_2(\bx,\bz)P(\bz)d\bz,$ where $J_1(\bx) \coloneqq \{\bz \in \R^d: \langle \bz,\bx\rangle \geq 0 \}$ and $M_2(\bx,\bz) = (\|\bx\|^2 + \langle \bz,\bx\rangle)\calN_2(\|\bx + \bz\|) + (\|\bx\|^2 - \langle \bz,\bx\rangle)\calN_2(\|\bx - \bz\|)$. Consider the set $J_2(\bx) \coloneqq \left\{\bz \in \R^d: \frac{\langle \bz,\bx\rangle}{\|\bz\|\|\bx\|} \in [0,0.5] \right\} \cup \{\mathbf{0}\}$. Clearly $J_2(\bx) \subset J_1(\bx)$. Note that on $J_1(\bx)$ we have $\|\bx + \bz \| \geq \|\bx - \bz\|$, which, together with the fact that $\calN_2$ is non-increasing, implies
    \begin{equation}\label{eq:identity}
        \calN_2(\|\bx - \bz\|) - \calN_2(\|\bx + \bz\|) = \left|\calN_2(\|\bx - \bz\|) - \calN_2(\|\bx + \bz\|) \right|, 
    \end{equation} for any $\bz \in J_2(\bx) \subset J_1(\bx)$. For any $\bz \in J_2(\bx)$ such that $\|\bz\| > \|\bx\|$, we then have
    \begin{align*}
        M_2&(\bz,\bx) = \|\bx\|^2[\calN_2(\|\bx - \bz\|) + \calN_2(\|\bx + \bz\|)] - \langle \bz,\bx\rangle[\calN_2(\|\bx - \bz\|) - \calN_2(\|\bx + \bz\|)] \\ &\stackrel{(a)}{=} \|\bx\|^2[\calN_2(\|\bx - \bz\|) + \calN_2(\|\bx + \bz\|)] - \langle \bz,\bx\rangle\left|\calN_2(\|\bx - \bz\|) - \calN_2(\|\bx + \bz\|)\right| \\ &\stackrel{(b)}{\geq} \|\bx\|^2[\calN_2(\|\bx - \bz\|) + \calN_2(\|\bx + \bz\|)] - \langle \bz,\bx\rangle\nicefrac{\|\bx\|}{\|\bz\|}[\calN_2(\|\bx - \bz\|) + \calN_2(\|\bx + \bz\|)] \\ &\stackrel{(c)}{\geq} 0.5\|\bx\|^2[\calN_2(\|\bx - \bz\|) + \calN_2(\|\bx + \bz\|)],
    \end{align*} where $(a)$ follows from \eqref{eq:identity}, $(b)$ follows from Lemma \ref{lm:jakovetic-joint}, while $(c)$ follows from the definition of $J_2(\bx)$. Next, consider any $\bz \in J_2(\bx)$, such that $0 < \|\bz\| \leq \|\bx\|$. We have
    \begin{align*}
        M_2&(\bz,\bx) = \|\bx\|^2[\calN_2(\|\bx - \bz\|) + \calN_2(\|\bx + \bz\|)] - \langle \bz,\bx\rangle[\calN_2(\|\bx - \bz\|) - \calN_2(\|\bx + \bz\|)] \\ &\stackrel{(a)}{=} \|\bx\|^2[\calN_2(\|\bx - \bz\|) + \calN_2(\|\bx + \bz\|)] - \langle \bz,\bx\rangle\left|\calN_2(\|\bx - \bz\|) - \calN_2(\|\bx + \bz\|)\right| \\ &\stackrel{(b)}{\geq} \|\bx\|^2[\calN_2(\|\bx - \bz\|) + \calN_2(\|\bx + \bz\|)] - 0.5\|\bx\|^2\left|\calN_2(\|\bx - \bz\|) - \calN_2(\|\bx + \bz\|)\right| \\ &\stackrel{(c)}{\geq} 0.5\|\bx\|^2[\calN_2(\|\bx - \bz\|) + \calN_2(\|\bx + \bz\|)],
    \end{align*} where $(a)$ again follows from \eqref{eq:identity}, $(b)$ follows from the definition of $J_2(\bx)$ and the fact that $0 < \|\bz\| \leq \|\bx\|$, while $(c)$ follows from the fact that $\calN_2$ is non-negative and the fact that $\left|\calN_2(\|\bx - \bz\|) - \calN_2(\|\bx + \bz\|)\right| \leq \calN_2(\|\bx - \bz\|) + \calN_2(\|\bx + \bz\|)$. Finally, if $\bz = \mathbf{0}$, we have $M_2(\mathbf{0},\bx) = 2\|\bx\|^2\calN_2(\|\bx\|) > 0.5\|\bx\|^2[\calN_2(\|\bx + \mathbf{0}\|) + \calN_2(\|\bx - \mathbf{0}\|)]$. Therefore, for any $\bz \in J_2(\bx)$, we have $M_2(\bz,\bx) \geq 0.5\|\bx\|^2[\calN_2(\|\bx - \bz\|) + \calN_2(\|\bx + \bz\|)] \geq \|\bx\|^2\calN_2(\|\bx\| + \|\bz\|)$, where the second inequality follows from the fact that $\calN_2$ is non-increasing and $\|\bx \pm \bz\| \leq \|\bx\| + \|\bz\|$. Combining everything, it readily follows that
    \begin{align}\label{eq:semi-done}
        \langle \bPhi(\bx),\bx \rangle \geq \int_{J_2(\bx)}M_2(\bx,\bz)P(\bz)d\bz \geq \|\bx\|^2 \int_{J_2(\bx)}\calN_2(\|\bx\| + \|\bz\|)P(\bz)d\bz.
    \end{align} Define $C_0 \coloneqq \min\left\{B_0,0.5 \right\}$ and consider the set $J_3(\bx) \subset J_2(\bx)$, defined as 
    \begin{equation*}
        J_3(\bx) \coloneqq \left\{\bz \in \R^d: \frac{\langle \bz,\bx\rangle}{\|\bz\|\|\bx\|} \in [0,0.5], \: \|\bz\| \leq C_0 \right\} \cup \{\mathbf{0}\}.
    \end{equation*} Since $a\calN_2(a)$ is non-decreasing, it follows that $\calN_2(a) \geq \calN_2(1)\min\left\{a^{-1},1 \right\}$, for any $a > 0$. For any $\bz \in J_3(\bx)$, it then holds that $\calN_2(\|\bz\| + \|\bx\|) \geq \calN_2(1)\min\left\{1/(\|\bx\|+C_0),1 \right\}$. Plugging in \eqref{eq:semi-done}, we then have
    \begin{align*}
        &\langle \bPhi(\bx),\bx \rangle \geq \|\bx\|^2 \int_{J_3(\bx)}\calN_2(\|\bx\| + \|\bz\|)P(\bz)d\bz \\ &\geq \|\bx\|^2\calN_2(1)\min\left\{(\|\bx\| + C_0)^{-1},1 \right\}\hspace{-0.5em}\int_{J_3(\bx)}\hspace{-2em}P(\bz)d\bz \geq \|\bx\|^2\calN_2(1)\min\left\{(\|\bx\| + C_0)^{-1},1 \right\}p_0.
    \end{align*} If $\|\bx\| \leq C_0$, it follows that $\|\bx\| + C_0 \leq 2C_0$, therefore $\min\left\{1/(\|\bx\|+C_0),1\right\} \geq \min\left\{1/(2C_0),1 \right\}$. Define $\kappa \coloneqq \min\left\{1/(2C_0),1 \right\}.$ If $\|\bx\| \geq C_0$, it follows that $\|\bx\| + C_0 \leq 2\|\bx\|$, therefore $\min\left\{1/(\|\bx\|+C_0),1 \right\} \geq \min\left\{1/(2\|\bx\|),1 \right\} \geq \min\left\{1/(2\|\bx\|),\kappa \right\}.$ Combining these facts, we get $\langle \bPhi(\bx),\bx \rangle \geq p_0\calN_2(1)\min\left\{\|\bx\|/2,\kappa\|\bx\|^2 \right\}$. Consider the constant $\kappa = \min\left\{\nicefrac{1}{(2C_0)},1 \right\}$. If $B_0 \geq 0.5$, it follows that $C_0 = 0.5$ and therefore $\kappa = 1$. On the other hand, if $B_0 < 0.5$, it follows that $C_0 = B_0$ and therefore $\kappa = \min\left\{\nicefrac{1}{(2B_0)},1 \right\} = 1$, as $2B_0 < 1$.
\end{proof}

\subsection{Proof of Theorem \ref{thm:non-conv} and Corollary \ref{cor:str-cvx}}\label{subsec:proof-thm-nonconv}

Prior to proving Theorem \ref{thm:non-conv}, we state two well-known results from the large deviation theory. The first result, known as Gartner-Ellis theorem, establishes conditions under which a family of random variables satisfies the LDP (or its upper bound), see, e.g., \cite{dembo2009large}.

\begin{lemma}\label{lm:gartner-ellis}
    Let $\Lambda_t: \R^d \mapsto \R$ be a sequence of log moment-generating functions associated to a given sequence of measures $\mu_t: \mathcal{B}(\R^d) \mapsto [0,1]$, $t \in \N$. If for some positive sequence $\{n_t\}_{t \in \N}$, such that $\lim_{t \rightarrow \infty}n_t = \infty$, and each $\bv \in \R^d$, we have $\limsup_{t \rightarrow \infty}\nicefrac{1}{n_t}\Lambda_t(n_t\bv) \leq \varphi(\bv) < \infty$, then the sequence $\{\mu_t\}_{t \in \N}$ satisfies the LDP upper bound with the rate function $I: \R^d \mapsto [0,\infty]$, given by the Fenchel-Legendre transform of $\varphi$, i.e., $I(\bx) = \varphi^\star(\bx) = \sup_{\bv \in \R}\left\{ \langle \bx,\bv\rangle - \varphi(\bv) \right\}$.
\end{lemma}

The Gartner-Ellis theorem provides a general framework for establishing LDP style upper bounds. A proof for the case $n_t = t$ can be found in, e.g., \cite{dembo2009large}. For completeness we provide a proof for the general positive decay rate $n_t$ in Appendix \ref{app:garnter-ellis}, built on the ideas of the case $n_t = t$. The next result, known as the contraction principle, shows how the LDP of one random variable can be used to establish the LDP for its continuous transformations.

\begin{lemma}\label{lm:contract-principle}
    Let $\mathcal{X}, \mathcal{Y}$ be Hausdorff topological spaces and $g: \mathcal{X} \mapsto \mathcal{Y}$ be a continuous function. Consider a good rate function $I: \mathcal{X} \mapsto [0,\infty]$. We then have:
    \begin{enumerate}[leftmargin=*,label=(a)]
        \item For each $y \in \mathcal{Y}$, define $I^\prime(y) = \inf\{I(x):  x \in \mathcal{X}, \: y = g(x) \}.$ Then $I^\prime$ is a good rate function on $\mathcal{Y}$, where the infimum over the empty set is taken to be $\infty$.

        \item If a sequence of random variables $\{X_t\}_{t \in \N}$ on $\mathcal{X}$ satisfies the LDP upper bound with rate function $I$, then the sequence of random variables $\{g(X_t) \}_{t \in \N}$ on $\mathcal{Y}$ satisfies the LDP upper bound with the rate function $I^\prime$.
    \end{enumerate}
\end{lemma}

As discussed in \cite{dembo2009large}, the rate function $I$ corresponding to the original sequence $\{X_t\}_{t \in \N}$ needs to be a good rate function, as otherwise, the function $I^\prime$ might not be a rate function at all. We are now ready to prove Theorem \ref{thm:non-conv}.

\begin{proof}[Proof of Theorem \ref{thm:non-conv}]
    Recall that $N_t \coloneqq \sum_{k = 1}^t\eta_t$ and define the quantity, $M_t \coloneqq \sum_{k = 1}^t\widetilde{\eta}_kG_k$, where $G_k \coloneqq \min\{\|\nabla f(\bxk) \|, \|\nabla f(\bxk)\|^2\}$. Combining the update rule \eqref{eq:nonlin-sgd2} and the smoothness inequality, we get
    \begin{align}
        f(\bxtp) &\leq f(\bxt) - \eta_t\langle \nabla f(\bxt),\bPhit - \bet \rangle + \nicefrac{\eta_t^2L}{2}\|\bPsi^{(t)}\|^2 \nonumber \\ &\leq f(\bxt) - \rho\eta_tG_t + \eta_t\langle \nabla f(\bxt),\bet\rangle + \nicefrac{\eta_t^2LC^2}{2}, \label{eq:L-smooth} 
    \end{align} where $\rho \coloneqq \min\{\alpha,\beta\}$ and the second inequality follows from Lemma \ref{lm:key-unified}. Summing up the first $t$ terms, rearranging and dividing both sides of \eqref{eq:L-smooth} by $N_t$, we get
    \begin{equation}\label{eq:part1}
        M_t \leq (\rho N_t)^{-1}\Big(f(\bx^{(1)}) - f^\star + \nicefrac{LC^2}{2}\sum_{k = 1}^{t}\eta_k^2\Big) + \rho^{-1}\sum_{k = 1}^{t}\widetilde{\eta}_k\langle \nabla f(\bxk), \bek \rangle. 
    \end{equation} Since $\sum_{k= 1}^t\widetilde{\eta}_k = 1$, it follows that $M_t \geq \min_{k \in [t]}G_k$. Define $Z_t \coloneqq \min_{k \in [t]}G_k$. Using \eqref{eq:part1}, we then have
    \begin{equation}\label{eq:5}
        Z_t \leq \underbrace{\rho^{-1}N_t^{-1}\Big(f(\bx^{(1)}) - f^\star + \nicefrac{LC^2}{2}\sum_{k = 1}^{t}\eta_k^2\Big)}_{\eqqcolon B_{1,t}} + \underbrace{\rho^{-1}\sum_{k = 1}^{t}\widetilde{\eta}_k\langle \nabla f(\bxk), \bek \rangle}_{\eqqcolon B_{2,t}}. 
    \end{equation} For any $\lambda \in \R$, define $\lambda_t \coloneqq n_t\lambda$, where $n_t > 0$ will be specified later. First, consider the case when $\lambda \geq 0$. We then have
    \begin{displaymath}
        \E\left[\exp(\lambda_t Z_t) \right] \stackrel{\eqref{eq:5}}{\leq} \E\left[\exp\left(\lambda_t(B_{1,t} + B_{2,t})\right) \right] = \exp\left(\lambda_tB_{1,t} \right)\E\left[\exp(\lambda_tB_{2,t})\right], 
    \end{displaymath} where the last equality follows by noticing that $B_{1,t}$ is a deterministic quantity. We now bound $\E[\exp(\lambda_tB_{2,t})]$. First, consider $\|\nabla f(\bxk)\|$, for any $k \geq 2$. {Using the triangle inequality, $L$-Lipschitz gradients of $f$ and update rule \eqref{eq:nonlin-sgd}, we get}
    \begin{align}
        \|\nabla f(&\bxk)\| \leq \|\nabla f(\bxk) - \nabla f(\bx^{(1)})\| + \|\nabla f(\bx^{(1)})\| \leq L\|\bxk - \bx^{(1)}\| + \|\nabla f(\bx^{(1)})\| \nonumber
        \\ &\leq L\left(\|\bx^{(k-1)} - \bx^{(1)}\| +  \eta_{k-1}C\right) + \|\nabla f(\bx^{(1)})\| \leq \ldots \leq LCN_k + \|\nabla f(\bx^{(1)})\|. \label{eq:part3}
    \end{align} Denote by $\E_t[\cdot] \coloneqq \E[ \cdot \: \vert \: \mathcal{F}_t]$ the expectation conditioned on history up to $t$. Then
    \begin{align*}
        \E&[\exp(\lambda_tB_{2,t})] = \E\Big[\exp\Big(\rho^{-1}\lambda_t\sum_{k = 1}^{t}\widetilde{\eta}_k\langle \nabla f(\bxk), \bek \rangle\Big)\Big] \nonumber \\ &= \E\Big[\exp\Big(\rho^{-1}\lambda_t\sum_{k = 1}^{t-1}\widetilde{\eta}_k\langle \nabla f(\bxk), \bek \rangle \Big)\E_{t}\Big[\exp(\rho^{-1}\lambda_t\widetilde{\eta}_{t}\langle \nabla f(\bx^{(t)}),\be^{(t)} \rangle) \Big] \Big] \nonumber \\ &\leq \E\Big[\exp\Big(\rho^{-1}\lambda_t\sum_{k = 1}^{t-1}\widetilde{\eta}_k\langle \nabla f(\bxk), \bek \rangle \Big) \exp\left(2\rho^{-2}C^2\lambda_t^2\widetilde{\eta}_{t}^2\|\nabla f(\bx^{(t)})\|^2 \right) \Big],
    \end{align*} where the inequality follows from Lemma \ref{lm:error_component}. Applying \eqref{eq:part3} in the above inequality, repeating the arguments recursively and combining everything, we get
    \begin{equation}
    \label{eq:ncvx-zt-mgf}
        \E\left[\exp(\lambda_tZ_t) \right] \leq \exp\Big(\lambda_tB_{1,t} + 4\rho^{-2}C^2\lambda_t^2\Big({\|\nabla f(\bx^{(1)})\|^2}\sum_{k = 1}^{t}\widetilde{\eta}_k^2 + L^2C^2\sum_{k = 1}^{t}\widetilde{\eta}_k^2N_k^2\Big) \Big).
    \end{equation} Denote the log moment-generating function of $Z_t$ by $\Lambda_t$, i.e., $\Lambda_t(\lambda) \coloneqq \log \E[\exp(\lambda Z_t)]$. Taking the logarithm and dividing by $n_t$, we then get
    \begin{equation}\label{eq:upper-b}
        n_t^{-1}\Lambda_t(n_t\lambda) \leq \lambda B_{1,t} + 4n_t^{-1}\rho^{-2}C^2\lambda^2\Big({\|\nabla f(\bx^{(1)})\|^2}\sum_{k = 1}^{t}\widetilde{\eta}_k^2 + L^2C^2\sum_{k = 1}^{t}\widetilde{\eta}_k^2N_k^2\Big).
    \end{equation} Consider the right-hand side of \eqref{eq:upper-b}. From the definition of $B_{1,t}$, we have
    \begin{align}
    \label{eq:darboux}
        \lambda B_{1,t} = \rho^{-1}\lambda N_t^{-1}\Big(f(\bx^{(1)}) - f^\star + \nicefrac{LC^2}{2}\sum_{k = 1}^{t}\eta_k^2\Big).
    \end{align} From the choice of step-size $\eta_t = \nicefrac{a}{(t + 1)^\delta}$, $\delta  \in (\nicefrac{1}{2},1)$, and the Darboux sums, we get $N_t \geq a[(t+1)^{1-\delta} - 1]/(1-\delta)$ and $\sum_{k = 1}^{t}\eta_k^2 \leq a^2/(2\delta - 1)$. Therefore, it readily follows that $\lambda B_{1,t} = \mathcal{O}(t^{\delta - 1})$, implying that $\limsup_{t \rightarrow \infty}\lambda B_{1,t} = 0$. For ease of notation, let $R_1 \coloneqq 4\rho^{-2}C^2{\|\nabla f(\bx^{(1)})\|^2}$ and $R_2 \coloneqq 4\rho^{-2}C^4L^2$. Consider three different step-sizes. 
    \begin{enumerate}
        \item $\delta \in (\nicefrac{1}{2},\nicefrac{3}{4})$. First, consider the term $n_t\sum_{k = 1}^t\widetilde{\eta}_k^2N_k^2$. We then have
        \begin{align}
            n_t\sum_{k = 1}^t\widetilde{\eta}_k^2N_k^2 = n_tN_t^{-2}\sum_{k = 1}^t\eta_k^2N_k^2 &\leq \nicefrac{n_ta^2}{[(t+1)^{1-\delta}-1]^2}\sum_{k = 1}^t(k+1)^{2-4\delta} \label{eq:step1}
        \end{align} Using the Darboux sums and $\delta \in (\nicefrac{1}{2},\nicefrac{3}{4})$, we get $\sum_{k = 1}^t(k+1)^{2-4\delta} \leq (t+1)^{3-4\delta}/(3-4\delta)$. Plugging back into \eqref{eq:step1}, we get $n_t\sum_{k = 1}^t\widetilde{\eta}_kN_k^2 \leq \frac{a^2n_t(t+1)^{3-4\delta}}{(3-4\delta)[(t+1)^{1-\delta} - 1]^2}$. Choosing $n_t = t^{2\delta-1}$, it readily follows that $\limsup_{t \rightarrow \infty}n_t\sum_{k = 1}^t\widetilde{\eta}_k^2N_k^2 \leq a^2/(3-4\delta)$. Next, consider $n_t\sum_{k = 1}^t\widetilde{\eta}_k^2$. We then have $n_t\sum_{k = 1}^t\widetilde{\eta}_k^2 \leq \frac{(1-\delta)^2t^{2\delta-1}}{(2\delta - 1)[(t+1)^{1-\delta}-1]^2} = \mathcal{O}\left(t^{4\delta - 3}\right)$. From the choice $\delta \in (\nicefrac{1}{2},\nicefrac{3}{4})$, it follows that $\limsup_{t \rightarrow \infty}n_t\sum_{k = 1}^t\widetilde{\eta}_k^2 = 0$. As such, we get $\limsup_{t \rightarrow \infty} t^{1-2\delta}\Lambda_t(t^{2\delta-1}\lambda) \leq a^2R_2\lambda^2/(3-4\delta) \eqqcolon \varphi_1(\lambda)$.

        \item $\delta = \nicefrac{3}{4}$. In this case, from \eqref{eq:step1}, we can bound the term $n_t\sum_{k = 1}^t\widetilde{\eta}_k^2N_k^2$ as $n_t\sum_{k = 1}^t\widetilde{\eta}_k^2N_k^2 \leq \frac{n_ta^2}{[(t+1)^{\nicefrac{1}{4}} - 1]^2}\sum_{k = 1}^t\frac{1}{k+1} \leq \frac{n_ta^2\ln(t+1)}{[(t+1)^{\nicefrac{1}{4}}-1]^2}$. Similarly, we can bound $n_t\sum_{k = 1}^t\widetilde{\eta}_k^2$ by $n_t\sum_{k = 1}^t\widetilde{\eta}_k^2 \leq \frac{n_t}{8[(t+1)^{\nicefrac{1}{4}} - 1]^2}$. Combining, it can be readily seen that, for the choice $n_t = \frac{\sqrt{t}}{\ln(t)}$, we get $\limsup_{t \rightarrow \infty}\nicefrac{\ln(t)}{\sqrt{t}}\Lambda_t(\nicefrac{\lambda\sqrt{t}}{\ln(t)}) \leq a^2R_2\lambda^2 \eqqcolon \varphi_2(\lambda)$.

        \item $\delta \in (\nicefrac{3}{4},1)$. Following the same steps as in the previous case, it can be readily verified that, for the choice $n_t = t^{2(1-\delta)}$, we have 
        \begin{equation*}
            \limsup_{t \rightarrow \infty}t^{2(\delta - 1)}\Lambda_t(t^{2(1-\delta)}\lambda) \leq \left((1 - \delta)^2R_1 / (2\delta - 1) + a^2R_2/(4\delta - 3) \right)\lambda^2 \eqqcolon \varphi_3(\lambda).
        \end{equation*}
    \end{enumerate} Next, consider any $\lambda < 0$. In this case, we can trivially bound the log MGF as $\Lambda_t(n_t\lambda) \leq 0$. Defining $\varphi: \R \times (\nicefrac{1}{2},1) \mapsto [0,\infty)$ as $\varphi(\lambda,\delta) \coloneqq \widetilde{\varphi}(\lambda,\delta)$, if $\lambda \geq 0$, otherwise $\varphi(\lambda,\delta) \coloneqq 0$, where $\widetilde{\varphi}: \R_{+} \times (\nicefrac{1}{2},1) \mapsto [0,\infty)$ is given by 
    \begin{displaymath}
        \widetilde{\varphi}(\lambda,\delta) = \begin{cases} \varphi_1(\lambda), & \delta \in (\nicefrac{1}{2},\nicefrac{3}{4}) \\ 
        \varphi_2(\lambda), & \delta = \nicefrac{3}{4} \\
        \varphi_3(\lambda), & \delta \in (\nicefrac{3}{4},1)
        \end{cases},     
    \end{displaymath} it follows that $\limsup_{t \rightarrow \infty}n_t^{-1}\Lambda(n_t\lambda) \leq \varphi(\lambda)$, satisfying the conditions of Lemma \ref{lm:gartner-ellis}. Therefore, $Z_t = \min_{k \in [t]}\min\{\| \nabla f(\bxk)\|, \| \nabla f(\bxk)\|^2 \}$ satisfies the LDP upper bound with the rate function $I(x) = \widetilde{\varphi}^\star(x,\delta)$, if $x \geq 0$, otherwise $I(x) = +\infty$, where 
    \begin{displaymath}
        \widetilde{\varphi}^\star(x,\delta) = \begin{cases} 
        \frac{(3-4\delta)x^2}{4a^2R_2}, & \delta \in (\nicefrac{1}{2},\nicefrac{3}{4}) \\
        \frac{x^2}{4a^2R_2}, & \delta = \nicefrac{3}{4} \\
        \frac{(2\delta-1)(4\delta-3)x^2}{4(1-\delta)^2(4\delta-3)R_1 + 4a^2(2\delta - 1)R_2}, & \delta \in (\nicefrac{3}{4},1)
    \end{cases}.
    \end{displaymath} It can be noted that $I$ has compact sub-level sets, making it a good rate function. To provide a LDP upper bound for $X_t = \min_{k \in [t]}\|\nabla f(\bxk) \|^2$, we proceed as follows. Let $Y_t \coloneqq \min_{k \in [t]}\|\nabla f(\bxk)\|$ and note that $Z_t = \min\{Y_t,X_t\}$. It can then be seen that $X_t = h(Z_t)$, where $h: [0,\infty] \mapsto [0,\infty]$ is given by $h(x) = \max\{x,x^2\}$. Therefore, defining the continuous function $g: \R \mapsto \R$, given by $g(x) = h(x)$, if $x \geq 0$, otherwise $g(x) = x$, we can apply Lemma \ref{lm:contract-principle}, establishing a LDP upper bound for $X_t = g(Z_t)$, with decay rate $n_t$ and good rate function $I^\prime(y) = \inf\{I(x): x \in \R, \: g(x) = y\}$.
\end{proof}  

\begin{proof}[Proof of Corollary \ref{cor:str-cvx}]
    We start by recalling the definition of the Huber loss function \cite{huber_loss}, $H_{\lambda}: \R \mapsto [0,\infty)$, parametrized by $\lambda > 0$ and given by 
    \begin{equation*}
        H_{\lambda}(x) \triangleq \begin{cases}
            \frac{1}{2}x^2, & |x| \leq \lambda, \\
            \lambda|x| - \frac{\lambda^2}{2}, & |x| > \lambda.
        \end{cases}
    \end{equation*} By the definition of Huber loss, it is not hard to see that it is a convex, non-decreasing function on $[0,\infty)$. Moreover, by the definition of Huber loss, we have, for any $k \geq 1$
    \begin{equation}\label{eq:7}
    \begin{aligned}
        G_k = \min\{\|\nabla f(\bx^{(k)})\|,\|\nabla f(\bx^{(k)})\|^2 \} \geq H_{1}(\|\nabla f(\bx^{(k)})\|).
    \end{aligned}
    \end{equation} Next, using the properties of strong convexity, it can be shown that strong convexity implies the gradient domination property, i.e., $\|\nabla f(\bx) \| \geq \mu\|\bx - \bx^\star\|$, for any $\bx \in \R^d$. Plugging it in \eqref{eq:7} and defining the Polyak-Ruppert average $\widehat{\bx}^{(t)} \coloneqq \frac{1}{t}\sum_{k = 1}^t\bxk$, we get 
    \begin{equation}\label{eq:Huber-ineq}
        \frac{1}{t}\sum_{k = 1}^{t} G_k \geq \frac{1}{t}\sum_{k = 1}^{t}H_{1}(\mu\|\bx^{(k)} - \bx^\star\|) \geq \mu^2H_{1/\mu}(\|\widehat{\bx}^{(t)} - \bx^\star\|),
    \end{equation} where the first inequality follows from~\eqref{eq:7}, the gradient domination property and the fact that $H$ is non-decreasing, while the second inequality follows from the fact that $H$ is convex and non-decreasing, applying Jensen's inequality twice and noticing that $H_{\lambda}(\mu x) = \mu^2H_{\lambda/\mu}(x)$. Next, instead of dividing both sides of the inequality in \eqref{eq:part1} with $N_t$, we instead divide by $\frac{1}{t}$ and use the fact that the sequence of step-sizes is decreasing, to get
    \begin{equation*}
        \frac{\eta_t}{t}\sum_{k = 1}^tG_k \leq \frac{1}{t}\sum_{k = 1}^t\eta_kG_k \leq \frac{1}{\rho t}\Big(f(\bx^{(1)}) - f^\star + \nicefrac{LC^2}{2}\sum_{k = 1}^{t}\eta_k^2 + \sum_{k = 1}^{t}\eta_k\langle \nabla f(\bxk), \bek\rangle\Big).
    \end{equation*} Dividing both sides of the above inequality by $\eta_t$ and using \eqref{eq:Huber-ineq}, we get
    \begin{equation*}
        H_{1/\mu}(\|\widehat{\bx}^{(t)} - \bx^\star\|) \leq \frac{2t^{\delta-1}}{a\rho\mu^{2}}\Big(f(\bx^{(1)}) - f^\star + \nicefrac{LC^2}{2}\sum_{k = 1}^{t}\eta_k^2 + \sum_{k = 1}^{t}\eta_k\langle \nabla f(\bxk), \bek\rangle\Big).
    \end{equation*} The proof now follows the exact same steps as in Theorem \ref{thm:non-conv}, establishing a LDP upper-bound for the quantity $Z_t \coloneqq H_{1/\mu}(\|\widehat{\bx}^{(t)} - \bx^\star\|)$, with the rate function $I(x) = \widetilde{\varphi}^\star(x,\delta)$, if $x \geq 0$, otherwise $I(x) = +\infty$, where 
    \begin{displaymath}
        \widetilde{\varphi}^\star(x,\delta) = \begin{cases} 
        \frac{(3-4\delta)x^2}{4a^2R_2}, & \delta \in (\nicefrac{1}{2},\nicefrac{3}{4}) \\
        \frac{x^2}{4a^2R_2}, & \delta = \nicefrac{3}{4} \\
        \frac{(2\delta-1)(4\delta-3)x^2}{4(1-\delta)^2(4\delta-3)R_1 + 4a^2(2\delta - 1)R_2}, & \delta \in (\nicefrac{3}{4},1)
    \end{cases},
    \end{displaymath} where $R_1 \coloneqq 16\rho^{-2}\mu^{-4}C^2\|\nabla f(\bx^{(1)})\|^2$ and $R_2 \coloneqq 16\rho^{-2}\mu^{-4}C^4L^2$. It can be noted that $I$ has compact sub-level sets, making it a good rate function. To provide a LDP upper bound for $X_t = \|\widehat{\bx}^{(t)} - \bx^\star \|^2$, we proceed as follows. By the definition of Huber loss, it is not hard to see that $X_t = h(Z_t)$, where $h: [0,\infty] \mapsto [0,\infty]$ is given by 
    \begin{equation*}
        h(x) = \begin{cases}
            2x, & x \leq \frac{1}{2\mu^2} \\
            \mu^2\big(x + \frac{1}{2\mu^2}\big)^2, & x > \frac{1}{2\mu^2}
        \end{cases}.
    \end{equation*} Therefore, defining the continuous function $g: \R \mapsto \R$, given by $g(x) = h(x)$, if $x \geq 0$, otherwise $g(x) = x$, we can apply Lemma \ref{lm:contract-principle}, establishing a LDP upper bound for $X_t = g(Z_t)$, with decay rate $n_t$ and good rate function $I^\prime(y) = \inf\{I(x): x \in \R, \: g(x) = y\}$.
\end{proof}

\subsection{Proofs of Theorems \ref{thm:mse-ft}, \ref{thm:cvx-mse} and Corollary \ref{cor:mse}}\label{subsec:proofs-mse}

\begin{proof}[Proof of Theorem \ref{thm:mse-ft}]
    We start from the inequality \eqref{eq:L-smooth}, which gives
    \begin{equation*}
        f(\bxtp) \leq f(\bxt) - \rho\eta_tG_t + \eta_t\langle \nabla f(\bxt), \bet \rangle + \eta_t^2LC^2/2,
    \end{equation*} where $G_t \coloneqq \min\{\| \nabla f(\bxt)\|, \|\nabla f(\bxt)\|^2\}$. Take the conditional expectation, to get
    \begin{equation*}
        \E[f(\bxtp) \: \vert \: \mathcal{F}_t] \leq f(\bxt) - \rho\eta_tG_t + \eta_t^2LC^2/2.
    \end{equation*} Take the full expectation, sum up the first $t$ terms and divide by $N_t$, to get
    \begin{equation*}\label{eq:important}
        \rho Z_t \leq \rho\sum_{k = 1}^t\widetilde{\eta}_k\E G_k \leq N_t^{-1}\Big(f(\bx^{(1)}) - f^\star + LC^2/2\sum_{k = 1}^t\eta_k^2\Big),
    \end{equation*} where $Z_t \coloneqq \min_{k \in [t]}\E\min\{\|\nabla f(\bxk)\|,\|\nabla f(\bxk)\|^2 \}$. The result now readily follows by considering $N_t^{-1}$ and $\sum_{k = 1}^t\eta_k^2$ for different choices of step-size.
\end{proof}

\begin{proof}[Proof of Corollary \ref{cor:mse}]
    Let $X_t = \min_{k \in [t]}\|\nabla f(\bxk)\|^2$, $Y_t = \min_{k \in [t]}\|\nabla f(\bxk)\|$ and $Z_t = \min\{X_t,Y_t\}$. It then follows that
    \begin{align*}
        \E[X_t] &= \E[X_t\bi_{\{Y_t \leq 1\}} ] + \E[X_t\bi_{\{Y_t > 1\}}] = \E[Z_t\bi_{\{Y_t \leq 1\}}] + \E[X_t\bi_{\{Y_t > 1\}}] \\ &\leq \E[Z_t] + \sqrt{\E[X_t^2]\mathbb{P}(Y_t > 1)} = \E[Z_t] + \sqrt{\E[X_t^2]\mathbb{P}(X_t > 1)} ,
    \end{align*} where $\mathbb{I}_{A}$ is the indicator of event $A$, we used the definitions of $X_t,Y_t,Z_t$ in the second and H\"{o}lder's inequality in the third inequality. From the definition of $X_t$, we have $X_t^2 = \min_{k \in [t]}\|\nabla f(\bxk)\|^4 \leq C_1t^{4(1-\delta)}$, for some $C_1 > 0$, where we used \eqref{eq:part3} in the last inequality. To bound the second term, i.e., $\mathbb{P}(X_t > 1)$, we use Theorem 1 from \cite{armacki2023high}, which states that, for any $\delta \in (2/3,3/4)$ and any $\epsilon > 0$, we have 
    \begin{equation*}
        \mathbb{P}(X_t > \epsilon) \leq \exp\left(-\frac{t^{1-\delta}}{C_2}\big(\epsilon - C_3t^{\delta-1} + C_4t^{2-3\delta}  \big)\right),    
    \end{equation*} where $C_2 = \frac{2(1-\delta)}{a\beta}$, $C_3 = \frac{2(1-\delta)}{a\beta}\Big(f(\bx^{(1)})-f^\star + a^2C^2(L/2 + 8\|\nabla f(\bx^{(1)})\|^2) \Big)$ and $C_4 = \frac{16a^3C^4L^2}{\beta(1-\delta)(3-4\delta)}$.\footnote{Note that Theorem 1 in \cite{armacki2023high} contains higher-order terms, e.g., of the order $t^{2(\delta-1)}$ and $t^{4-6\delta}$, which are ignored, as they are not the dominating terms and for the ease of exposition. Also note that the constant $C_3$ in \cite{armacki2023high} contains the quantity $\inf_{\{\bx^\star: \nabla f(\bx^\star) = 0\}}\|\bx^{(1)}-\bx^\star\|^2$, however, this quantity can be replaced by $\|\nabla f(\bx^{(1)})\|^2$, following the same approach as in the proof of Theorem \ref{thm:non-conv}.} Noticing that $\delta-1,2-3\delta < 0$, since $\delta \in (2/3,3/4)$, it then follows that, for any $t \geq \max\{\sqrt[1-\delta]{4C_3},\sqrt[3\delta-2]{4C_4} \}$ and any $\delta \in (2/3,3/4)$, we have
    \begin{equation*}
        \mathbb{P}(X_t > 1) \leq \exp\left(-\nicefrac{t^{1-\delta}}{2C_2}\right),    
    \end{equation*} Combining everything, we get, for any $t \geq \max\{\sqrt[1-\delta]{4C_3},\sqrt[3\delta-2]{4C_4} \}$ and any $\delta \in (2/3,3/4)$
    \begin{equation*}
        \E[X_t] \leq \E[Z_t] + \sqrt{C_1}t^{2(1-\delta)}\exp(-\nicefrac{t^{1-\delta}}{2C_2}).
    \end{equation*} Using Theorem \ref{thm:mse-ft} to bound $E[Z_t]$, for the range $\delta \in (2/3,3/4)$, and noticing that this term is dominating, the result now readily follows.
\end{proof}

\begin{proof}[Proof of Theorem \ref{thm:cvx-mse}]
    From Lemma \ref{lm:key-unified}, we have
    \begin{align}
    \label{eq:gd-related-bd1}
    \begin{split}
        \E \Big[ \langle\nabla f(\bxt), & \boldsymbol{\Phi}^{(t)}\rangle  \Big] 
        \ge  \E \Big[ \langle\nabla f(\bxt), \boldsymbol{\Phi}^{(t)}\rangle  \bi_{\{\| \nabla f(\bxt) \|  \leq \alpha/\beta\}} \Big] \\ &= \beta \E \Big[ \|\nabla f(\bxt) \|^2 \bi_{\{\| \nabla f(\bxt ) \| \leq \alpha/\beta \}} \Big]
         \\ &= \beta  \Big( \E \Big[ \|\nabla f(\bxt) \|^2 \Big] -   \E \Big[ \|\nabla f(\bxt) \|^2 \bi_{\{\| \nabla f(\bxt) \| > \alpha/\beta\}} \Big] \Big),
    \end{split}
    \end{align} where $\bi_A$ is the indicator of event $A$. We now estimate the second expectation from the last line of the above relation. Using Hölder's inequality, we get 
    \begin{align}
    \label{eq:gd-related-bd2}
    \begin{split}
        \E \Big[ \|\nabla f(\bxt) \|^2 \bi_{\{\| \nabla f(\bxt) \| > \alpha/\beta \}} \Big] \leq \sqrt{ \E \big[ \| \nabla f(\bxt) \|^4 \big] \mathbb{P} \big( \| \nabla f(\bxt) \| > \alpha/\beta \big) }
    \end{split}
    \end{align}
    From Theorem 2 in \cite{armacki2023high}, convexity and $L$-smoothness of $f$, it follows that 
    \begin{equation}\label{eq:gd-tail-bd}
        \mathbb{P} \Big(\|\nabla f(\bx  ^t) \| > \alpha/\beta \Big) 
        \leq \mathbb{P} \Big( f(\bxt) - f^\star > \alpha^2/2\beta^2L \Big) \leq ee^{-\frac{B \alpha^2}{2L\beta^2} (t + 1)^{\min \{2a \mu \nu, 2\delta - 1\}}},
    \end{equation}
    for some problem related constants $\nu, B > 0$. Taking the full expectation in the first line in \eqref{eq:L-smooth} and combining with \eqref{eq:gd-related-bd1}, \eqref{eq:gd-related-bd2} and \eqref{eq:gd-tail-bd}, we have
    \begin{align*}
        \E[f(\bxtp)] &\leq \E[f(\bxt)] - \beta\eta_t\E \big[ \|\nabla f(\bxt) \|^2\big] + \frac{\eta^2_tLC^2}{2} \\ &+  \beta\eta_t e^{\frac{1}{2}-\frac{B \alpha^2}{4L\beta^2} (t + 1)^{\min \{2a \mu \nu, 2\delta - 1\}}}\sqrt{ \E \big[ \| \nabla f(\bxt) \|^4 \big] }  
    \end{align*} Next, from \eqref{eq:part3} and the choice of step-size, we know that $\|\nabla f(\bxt)\|^4 \leq C_1t^{4(1-\delta)}$, where $C_1 > 0$ is some problem related constant. Combining everything, we finally get 
    \begin{align*}
        \E[f(\bx^{t + 1})] 
        & \leq \E[f(\bxt)] - \frac{a \beta } {(t+1)^\delta} \E[\| \nabla f(\bxt) \|^2 ] + \frac{a^2 L C^2}{2 ( t+ 1)^{2\delta}}  \\
        & + a \beta C_2 (t+1)^{2 -3\delta} \exp  \left( \frac{1}{2} -\frac{B \alpha^2}{4L\beta^2 } (t + 1)^ {\min \{ 2a \mu \nu, 2\delta - 1\} } \right),
    \end{align*} where $C_2 > 0$ is some problem related constant. Using $\mu$-strong convexity and Lemma 4 (pp. 45) in \cite{polyak1987introduction} in the above relation, we obtain the desired result. 
\end{proof}

\subsection{Proof of Theorem \ref{thm:ncvx-as}}\label{subsec:proofs-as}

\begin{proof} 
    We first prove the claim for non-convex costs. Define
    \begin{align*}
        \lambda_t(\delta) = \begin{cases}
            (t+1)^{\delta - \nicefrac{1}{2}}, & \delta \in (\nicefrac{1}{2}, \nicefrac{3}{4})  \\
            (t+1)^{\nicefrac{1}{4} - \epsilon}, & \delta = \nicefrac{3}{4} \text{ and } \epsilon \in (0, \nicefrac{1}{4})  \\ 
            (t+1)^{1 - \delta}, & \delta \in (\nicefrac{3}{4}, 1) 
        \end{cases}.
    \end{align*} We will prove that $\lambda_t(\delta) (t+1)^{-\epsilon_1} Z_t \xrightarrow{\text{a.s.}}  0$, for all $\epsilon_1 < \min\{ 1-\delta, \delta-\nicefrac{1}{2}, \nicefrac{1}{4}-\epsilon\}$, where $Z_t = \min_{k \in [t]}\min\{\|\nabla f(\bxk)\|,\|\nabla f(\bxk)\|^2 \}$. We begin by showing that $\E \big[e^{\lambda_t(\delta)Z_t}\big] \leq C_1$, for some $C_1 > 0$. From \eqref{eq:ncvx-zt-mgf} and using $\lambda_t = \lambda_t(\delta)$, it follows that   
    \begin{align*}
     \E\big[ \exp(\lambda_t(\delta) Z_t) \big]
    & \le  \exp\Big[  (\rho N_t)^{-1}  \Big( f(\bx^{(1)}) - f^\star + LC^2/2\sum_{k = 1}^{t}\eta_k^2 \Big) \lambda_t(\delta)  \\
    & + 4 \rho^{-2} C^2 \Big( {\|\nabla f(\bx^{(1)})\|^2}N_t^{-2}\sum_{k=1}^t \eta_k^2  + L^2C^2 \sum_{k=1}^t \widetilde{\eta}_k^2 N_k^2  \Big) \lambda_t^2(\delta) \Big]
    \end{align*}
    Substituting the bounds from \eqref{eq:darboux} and \eqref{eq:step1} into the above relation gives
    \begin{align*}
        \ln \E & \big[ e^{\lambda_t(\delta) Z_t} \big]
        \le \underbrace{\frac{(1-\delta)\lambda_t(\delta)}{a\rho[(t + 1)^{1 - \delta} - 1]}}_{\eqqcolon R_{1, t}} \left( f(\bx^{(1)}) - f^\star + \frac{a^2 L C^2}{2(2 \delta - 1) }\right)  \\ 
        & + \underbrace{\frac{4 C^2 {\|\nabla f(\bx^{(1)})\|^2} (1-\delta)^2\lambda_t^2 (\delta)}{\rho^2(2\delta-1)[(t + 1)^{1- \delta} - 1]^2}}_{\eqqcolon R_{2, t}} + \underbrace{\frac{4 a^2 C^4 L^2\lambda_t^2(\delta) \sum_{k=1}^t (k + 1)^{2 - 4\delta}}{\rho^{2}[(t + 1)^{1 - \delta} - 1]^2}}_{\eqqcolon R_{3, t}}.
    \end{align*} Consider different choices of $\delta$. For $\delta \in (\nicefrac{1}{2}, \nicefrac{3}{4})$, one has $\delta - \nicefrac{1}{2} < 1 - \delta$, implying $R_{1,t}, R_{2, t} < \infty$, as well as  $\sum_{k=1}^t (k+1)^{2-4\delta} \le (t+1)^{3 - 4\delta}/(3 - 4\delta)$, thus $R_{3, t} < \infty$. For $\delta = \nicefrac{3}{4}$, we have $R_{1, t}, R_{2,t} \rightarrow 0$ and $\sum_{k=1}^t (k+1)^{2-4\delta} \le \ln (t + 1)$, hence $R_{3, t} = \mathcal{O}((t+1)^{-2\epsilon} \ln (t + 1)) \rightarrow 0$. For $\delta \in (\nicefrac{3}{4}, 1)$, $R_{1, t}, R_{2,t} = \mathcal{O}(1)$ and $\sum_{k=1}^t (k+1)^{2 - 4\delta} \le 1/(4\delta - 3)$, therefore $R_{3, t} = \mathcal{O}(1)$. Combing the three cases, it readily follows that $\E [e^{\lambda_t(\delta) Z_t}] \leq C_1$, for some constant $C_1 > 0$. Next, fix $\epsilon_1 < \min\{ 1-\delta, \delta-\nicefrac{1}{2}, \nicefrac{1}{4}-\epsilon\}$, $\varepsilon > 0$ and define the event $E_t^\varepsilon = \{ \lambda_t(\delta)(t+1)^{-\epsilon_1}Z_t \ge \varepsilon \}$. We then have $\sum_{t = 1}^\infty \mathbb{P}(E_t^{\varepsilon}) \le \sum_{t = 1}^\infty e^{-\varepsilon(t+1)^{\epsilon_1}} \E \big[e^{\lambda_t(\delta)Z_t}\big] < \infty.$ By Borel-Cantelli lemma, for any $\varepsilon > 0$, there exists some $T_\varepsilon$, such that $\forall t \ge T_\varepsilon$, we have $\lambda_t(\delta)(t+1)^{-\epsilon_1} Z_t < \varepsilon$ almost surely, proving the first part. For the second part, first take the conditional expectation in \eqref{eq:L-smooth}, to get
    \begin{align}
    \label{eq:cvx-main-cond}
    \begin{split}
        \E \big[ f (\bxtp) \mid \mathcal{F}_{t} \big] 
        & \leq f(\bxt) - \eta_t\langle \nabla f(\bxt), \bPhi^{(t)} \rangle + \eta_t^2LC^2/2. 
    \end{split}
    \end{align}
    From Lemma 3.4 in \cite{armacki2023high}, there exists some constant $\nu > 0$, such that  
    \begin{align}
    \label{eq:gd_acute_angle}
        \langle \nabla f(\bxt), \bPhi^{(t)} \rangle \ge \nu ( t+ 1)^{\delta- 1} \| \nabla f(\bx^{(t)} \|^2. 
    \end{align} Combing \eqref{eq:cvx-main-cond},  \eqref{eq:gd_acute_angle} and $\mu$-strong convexity, it follows that
    \begin{align}
    \label{eq:cvx-mse-mrec}
        \E[f(\bxtp) - f^\star \mid \mathcal{F}_t ] \le \Big(1 - \frac{2a \mu \nu }{t + 1} \Big) \big(f(\bxt) - f^\star \big) + \frac{a^2 LC^2}{2(t + 1)^{2\delta}}.
    \end{align} 
    Let $\tau < \min \{2a \mu \nu, 2\delta - 1\}$. Since $\tau \le 1$,  $(t + 2)^\tau$ is concave in $t$ and we have $(t + 2)^\tau \le (t + 1)^\tau + \tau (t + 1)^{\tau  - 1} = (t + 1)^\tau[ 1 + \tau(t + 1)^{-1}]$, as well as $(t + 2)^\tau = (t + 1)^\tau\big( 1 + \frac{1}{t + 1}\big)^\tau \le e^\tau (t + 1)^\tau$. Using the two relations, letting $v^t = (t + 1)^\tau(f(\bxt) - f^\star)$, and multiplying both sides of \eqref{eq:cvx-mse-mrec} by $(t + 2)^\tau$, we obtain $\E [v^{t + 1} \mid \mathcal{F}_t] \le \left(1 - \frac{2a\mu \nu - \tau}{t + 1}\right) v^t + \frac{e^\tau a^2 LC^2}{2(t + 2)^{2\delta - \tau}}$. Applying Lemma 10 from \cite{polyak1987introduction} (pp. 49-50) leads to the desired result.  
\end{proof}

\section{Proof of Generalized Gartner-Ellis Theorem}\label{app:garnter-ellis} In this section we provide a proof of Lemma \ref{lm:gartner-ellis}. We begin by defining the concept of \emph{exponential tightness}.

\begin{definition}
    Let $\cal X$ be a topological space and $\cal B$ be the associated Borel $\sigma$-algebra. Assume that $\cal B$ contains all the compact subsets of $\cal X$. A family of probability measures $\{\mu_\epsilon \}$ on $\cal X$ is exponentially tight if for every $\alpha < +\infty$, there exists a compact set $K_\alpha \subset \cal X$, such that $\limsup_{\epsilon \rightarrow 0} \epsilon\log \mu_\epsilon(K^c_\alpha) < -\alpha$.
\end{definition}

Next, we state an intermediate result used in the proof. The proof is omitted and can be found in, e.g., \cite[Lemma 1.2.18]{dembo2009large}. 

\begin{lemma}\label{lm:intermed}
    Let $\{\mu_t \}_{t \in \N}$ be an exponentially tight family. If the large deviations upper-bound holds for all compact sets, then it also holds for all closed sets.
\end{lemma}

We are now ready to prove Lemma \ref{lm:gartner-ellis}.

\begin{proof}[Proof of Lemma \ref{lm:gartner-ellis}]
    The proof largely follows the idea of Cramer's theorem, e.g., \cite[Theorem 2.2.30]{dembo2009large}. In order to establish the LDP upper bound, it suffices to show that, for every $\delta > 0$ and closed set $F \subset \R^d$ 
    \begin{equation*}
        \limsup_{t \rightarrow \infty} \frac{1}{n_t}\log \mu_t(F) \leq \delta - \inf_{\bx \in F}I^{\delta}(\bx),
    \end{equation*} where $I^\delta(\bx) = \min\left\{\varphi^\star(\bx) - \delta, \frac{1}{\delta} \right\}$ is the truncated $\delta$-rate function, see, e.g., \cite{dembo2009large}. We first prove the inequality for compact sets. For a compact set $G \subset \R^d$ and any $\bq \in G$, choose $\bv_{\bq} \in \R^d$, so that $\langle\bq,\bv_{\bq}\rangle - \varphi(\bv_{\bq}) \geq I^\delta(\bq)$, noting that such choice is possible, since $I^\delta(\bq) \leq \varphi^\star(\bq) = \sup_{\bv \in \R^d}\left\{\langle \bq,\bv\rangle - \varphi(\bv) \right\}$. Next, for each $\bq$, choose $\rho_{\bq} > 0$, such that $\rho_{\bq}\|\bv_{\bq}\| \leq \delta$, and define $B_{\bq} = \{\bx: \|\bx - \bq\| \leq \rho_{\bq}\}$. For $t \in \N$, $\bv \in \R^d$ and measurable set $A \subset \R^d$, we then have 
    \begin{equation*}
        \mu_t(A) = \mathbb{P}(\bX_t \in A) \leq \mathbb{P}\left(\langle \bv, \bX_t\rangle \geq c_{\bv,A}\right) \leq \exp\left(-c_{\bv,A}\right)\E\left[\exp\left( \langle\bv, \bX_t\rangle \right) \right],    
    \end{equation*} where $c_{\bv,A} = \inf_{\by \in A} \langle\bv, \by\rangle$. Defining $c_{\bv_{\bq},B_{\bq}} = \inf_{\bx \in B_{\bq}} \langle\bv_{\bq},\bx\rangle$ and applying the inequality to $B_{\bq}$, we get, for any $\bq \in G$ 
    \begin{equation*}
        \mu_{t}(B_{\bq}) \leq \exp\left(-n_tc_{\bv_{\bq},B_{\bq}}\right)\E\left[\exp\left( n_t\langle\bv_{\bq}, \bX_t\rangle \right) \right].    
    \end{equation*} Similarly, we have
    \begin{equation*}
        -\inf_{\bx \in B_{\bq}}\langle\bv_{\bq},\bx\rangle = -\inf_{\bx \in B_{\bq}}\langle\bv_{\bq},\bx - \bq\rangle - \langle \bq,\bv_{\bq}\rangle \leq \delta - \langle \bq,\bv_{\bq}\rangle,
    \end{equation*} where the last inequality follows from $\|\bx - \bq\|\leq \rho_{\bq}$, for any $\bx \in B_{\bq}$ and $\rho_{\bq}\|\bv_{\bq}\| \leq \delta$. Using this bound, we then get
    \begin{equation}\label{eq:upper_bdd}
        \frac{1}{n_t}\log\mu_{t}(B_{\bq}) \leq \delta - \langle\bq,\bv_{\bq}\rangle + \frac{1}{n_t}\Lambda(n_t\bv_{\bq}).
    \end{equation} Since the set $G$ is compact, we can extract a finite sub-cover of $\cup_{\bq \in G}B_{\bq}$, e.g., $B_{\bq_i}$, $i = 1,\ldots,N$, for some $N \geq 1$, therefore 
    \begin{equation*}
        \frac{1}{n_t}\log\mu_{t}(G) \leq \frac{1}{n_t}\log\mu_{t}(\cup_{i = 1}^NB_{\bq_i}) \leq \frac{1}{n_t}\log\Big(\sum_{i = 1}^N\mu_t(B_{\bq_i})\Big).
    \end{equation*} Combining everything, we get
    \begin{align*}
        \limsup_{t \rightarrow \infty}&\frac{1}{n_t}\log\mu_t(G) \leq \limsup_{t \rightarrow \infty}\frac{1}{n_t}\log\Big(\sum_{i = 1}^N\mu_t(B_{\bq_i})\Big) = \max_{i \in [N]}\limsup_{t \rightarrow \infty}\frac{1}{n_t}\log\mu_t(B_{\bq_i}) \\ &\leq \max_{i \in [N]}\delta - \langle \bq_i,\bv_{\bq_i} \rangle + \limsup_{t \rightarrow\infty}\frac{1}{n_t}\Lambda(n_t\bv_{\bq_i}) \leq \max_{i \in [N]}\delta - \langle\bq_i,\bv_{\bq_i}\rangle + \varphi(\bv_{\bq_i}),
    \end{align*} where we used the fact that $\limsup_{\epsilon \rightarrow 0}\epsilon\log\left(\sum_{i = 1}^Na^i_\epsilon\right) = \max_{i \in [n]}\limsup_{\epsilon \rightarrow 0}\epsilon\log a^i_\epsilon$, we use \eqref{eq:upper_bdd} in the second inequality and use $\limsup_{t \rightarrow\infty}\nicefrac{1}{n_t}\Lambda_t(n_t\bv) \leq \varphi(\bv)$ in the last one. Rearrange, to get 
    \begin{equation*}
        \limsup_{t \rightarrow \infty}\frac{1}{n_t}\log\mu_t(G) \leq \delta - \min_{i \in [N]}\langle\bq_i,\bv_{\bq_i}\rangle - \varphi(\bv_{\bq_i}) \leq \delta - \min_{i \in [N]}I^\delta(\bq_i),   
    \end{equation*} where the last inequality follows from the choice of $\bv_{\bq_i}$. Since $\bq_i \in G$, taking the infimum over the entire $G$ gives the result. 

    Next, in order to extend the result to closed sets, Lemma \ref{lm:intermed} states that it suffices to show that $\{\mu_t\}_{t \in \N}$ is exponentially tight. To that end, let $\be_i, \: i = 1,\ldots,d$ denote the unit vectors in $\R^d$. Since $\varphi(\bv) < \infty$ for all $\bv \in \R^d$, it follows that there exist $\theta_i, \pi_i > 0$ such that $\varphi(\theta_i\be_i) < \infty$ and $\varphi(-\pi_i\be_i) < \infty$, for all $i = 1,\ldots,d$. Therefore,
    \begin{align*}
        \mu_t^i\left((-\infty,-\rho]\right) = \mathbb{P}\left(X^i_t \leq -\rho \right) &\leq \exp\left(-n_t\pi_i\rho\right)\E\left[\exp(-n_t\pi_iX^i_t) \right] \\ &= \exp\left(-n_t\pi_i\rho + \Lambda_t(-n_t\pi_i\be_i) \right),
    \end{align*} and, similarly,
    \begin{equation*}
        \mu_t^i\left([\rho,\infty)\right) \leq \exp\left(-n_t\theta_i\rho + \Lambda_t(n_t\theta_i\be_i) \right),
    \end{equation*} for all $i = 1,\ldots,d$, where $\mu^i_t$ are the laws of (the marginals of) coordinates of the random vector $\bX_t = \begin{bmatrix}X_t^1 & \ldots & X_t^d \end{bmatrix}^\top$. Hence, for any $i = 1,\ldots,d$, we have
    \begin{align*}
        \lim_{\rho \rightarrow \infty}\limsup_{t \rightarrow \infty}\frac{1}{n_t}\log\mu^i_t((-\infty,-\rho]) &= \lim_{\rho \rightarrow \infty}\limsup_{t \rightarrow \infty}\Big( -\pi_i\rho + \frac{1}{n_t}\Lambda_t(-n_t\pi_i\be_i)\Big) \\ &\leq -\infty + \varphi(-\pi_i\be_i) = -\infty.
    \end{align*} Similarly, we have, for any $i = 1,\ldots,d$
    \begin{equation*}
        \lim_{\rho \rightarrow \infty}\limsup_{t \rightarrow \infty}\frac{1}{n_t}\log\mu^i_t((-\infty,-\rho]) = -\infty.
    \end{equation*} Therefore, combining everything, we have
    \begin{align*}
        \lim_{\rho \rightarrow \infty}&\limsup_{t \rightarrow \infty}\frac{1}{n_t}\log\mu_t\left(([-\rho,\rho]^d)^c\right) \\ &\leq \lim_{\rho \rightarrow \infty}\limsup_{t \rightarrow \infty}\frac{1}{n_t}\log\left(\sum_{i = 1}^d\left[\mu_t^i((-\infty,-\rho]) + \mu_t^i([\rho,\infty))\right]\right) = -\infty,
    \end{align*} implying that $\{\mu_t\}_{t \in \N}$ is an exponentially tight sequence of measures.
\end{proof}

\end{document}